\documentclass[twoside]{article}

\usepackage[accepted]{aistats2020}
\usepackage[round]{natbib}

\usepackage{amsmath,amssymb,amsfonts,amsthm}
\usepackage{graphicx}
\usepackage[dvipsnames]{xcolor}
\usepackage{hyperref}
\usepackage{tikz}
\usepackage{dsfont}
\usetikzlibrary{arrows,calc,intersections}
\usepackage{bbm}
\usepackage{enumitem}
\usepackage{pifont}
\usepackage{fontawesome}

\newcommand{\A}{\mathcal{A}}
\newcommand{\calS}{\mathcal{S}}
\newcommand{\one}[1]{\mathds{1}_{\left(#1 \right)}}
\newcommand{\ucrl}{{\small\textsc{UCRL2}}\xspace}
\newcommand{\cucrl}{{\small\textsc{CUCRL2}}\xspace}
\newcommand{\evi}{{\small\textsc{EVI}}\xspace}
\newcommand{\rmaxbound}{r_{\max}}
\newcommand{\euler}{\textsc{EULER}\xspace}
\newcommand{\ucb}{\textsc{UCB}\xspace}
\newcommand{\cucb}{\textsc{CUCB}\xspace}
\newcommand{\ucbvi}{\textsc{UCB-VI}\xspace}
\newcommand{\cucbvi}{\textsc{CUCB-VI}\xspace}
\newcommand{\wt}[1]{\widetilde{#1}}
\newcommand{\wh}[1]{\widehat{#1}}

\newcommand{\wb}[1]{\overline{#1}}
\newcommand{\wu}[1]{\underline{#1}}

\usepackage{xspace}
\DeclareRobustCommand{\eg}{e.g.,\@\xspace}
\DeclareRobustCommand{\ie}{i.e.,\@\xspace}

\DeclareRobustCommand{\wrt}{w.r.t.\@\xspace}

\newcommand{\SR}{\Pi^\text{SR}}
\newcommand{\SD}{\Pi^\text{SD}}
\newcommand{\MR}{\Pi^\text{MR}}
\newcommand{\MD}{\Pi^\text{MD}}

\newcommand{\worstD}{\Upsilon}

\DeclareMathOperator*{\argmax}{\arg\,\max}

\newcommand{\transp}{\mathsf{T}}

\newlength{\minipagewidth}
\newlength{\minipagewidthx}
\newcommand{\bookboxx}[1]{\small
\par\medskip\noindent
\framebox[0.99\textwidth]{
\begin{minipage}{0.97\dimexpr\textwidth-\parindent\relax} {#1} \end{minipage} } \par\medskip }

\usepackage[colorinlistoftodos, textwidth=20mm]{todonotes}
\tikzstyle{every picture}+=[remember picture]

\newtheorem{lemma}{Lemma}

\newtheorem{assumption}{Assumption}

\newtheorem{proposition}{Proposition}

\definecolor{tab101}{HTML}{1f77b4}
\definecolor{tab102}{HTML}{ff7f0e}
\definecolor{tab103}{HTML}{2ca02c}

\usepackage[normalem]{ulem}
\begin{document}

%

%

\twocolumn[

\aistatstitle{Conservative Exploration in Reinforcement Learning}

\aistatsauthor{  Evrard Garcelon \And Mohammad Ghavamzadeh \And Alessandro Lazaric \And Matteo Pirotta}

\aistatsaddress{Facebook AI Reasearch } ]

\begin{abstract}
        While learning in an unknown Markov Decision Process (MDP), an agent should trade off exploration to discover new information about the MDP, and exploitation of the current knowledge to maximize the reward.
       Although the agent will eventually learn a good or optimal policy, there is no guarantee on the quality of the intermediate policies.
        This lack of control is undesired in real-world applications where a minimum requirement is that the executed policies are guaranteed to perform at least as well as an existing baseline.
        In this paper, we introduce the notion of conservative exploration for average reward and finite horizon problems. We present two optimistic algorithms that guarantee (w.h.p.) that the conservative constraint is never violated during learning. We derive regret bounds showing that being conservative does not hinder the learning ability of these algorithms.
\end{abstract}


\vspace{-0.1in}
\section{Introduction}
\vspace{-0.1in}
\label{sec:intro}
While Reinforcement Learning (RL) has achieved tremendous successes in simulated domains, its use in real system is still rare.
A major obstacle is the lack of guarantees on the learning process, that makes difficult its application in domains where hard constraints (\eg on safety or performance) are present.
Examples of such domains are digital marketing, healthcare, finance, and robotics.
For a vast number of domains, it is common to have a known and reliable \emph{baseline policy} that is potentially suboptimal but satisfactory. Therefore, for applications of RL algorithms, it is important that are guaranteed to perform at least as well as the existing baseline.

In the offline setting, this problem has been studied under the name of \emph{safety w.r.t.\ a baseline}~\citep{Bottou13CR,Thomas15HCO,Thomas15HCP,Swaminathan15CR,Petrik16SP,LarocheTC19,SimaoS19}.
Given a set of trajectories collected with the baseline policy, these approaches aim to learn a policy --without knowing or interacting with the MDP-- that is guaranteed (\eg w.h.p.) to perform at least as good as the baseline.
This requires that the set of trajectories is sufficiently reach in order to allow to perform counterfactual reasoning with it.
This often implies strong requirements on the ability of exploration of the baseline policy.
These approaches can be extended to a semi-batch settings where phases of offline learning are alternated with the executing of the improved policy.
This is the idea behind conservative policy iteration~\citep[\eg][]{KakadeL02,PirottaRPC13} where the goal is to guarantee a monotonic policy improvement in order to overcome the policy oscillation phenomena~\cite{bertsekas2011approximate}. These approaches has been successively extended to function approximation preserving theoretical guarantees~\cite[\eg][]{PirottaRB13,AchiamHTA17}.
A related problem studied in RL is the one of safety, where the algorithm is forced to satisfy a set of constraints, potentially not directly connected with the performance of a policy~\citep[\eg][]{altman1999constrained, BerkenkampTS017,ChowNDG18}.

In the online setting, which is the focus of this paper, the learning agent needs to trade-off exploration and exploitation while interacting with the MDP.
Opposite to offline learning, the agent has direct control over exploration.
Exploration means that the agent is willing to give up rewards for policies improving his knowledge of the environment.
Therefore, there is no guarantee on the performance of policies generated by the algorithm, especially in the initial phase where the uncertainty about the MDP is maximal and the algorithm has to explore multiple options (almost randomly).
To increase the application of exploration algorithm, it is thus important that the policies selected by the algorithm are (cumulatively) guaranteed to perform as well as the baseline by making exploration more \emph{conservative}.
This setting has been studied in multi-armed bandits~\citep{WuSLS16}, contextual linear bandits~\citep{KazerouniGAR17}, and stochastic combinatorial semi-bandits~\citep{Katariya2019interleaving}. 
These papers formulate the problem using a constraint defined based on the performance of the baseline policy (mean of the baseline arm in the multi-armed bandit case), and modify the corresponding UCB-type algorithm~\citep{auer2002finite} to satisfy this constraint.
Another algorithm in the online setting is by~\citep{mansour2015bayesian} that balances exploration and exploitation such that the actions taken are compatible with the agent's (customer's) incentive formulated as a Bayesian prior.

While the conservative exploration problem is well-understood in bandits, little is known about this setting to RL, where the actions taken by the learning agent affect the system state.
This dynamic component makes the definition of the conservative condition much less obvious in RL.
While in the bandit case it is sufficient to look at (an estimate of) the immediate reward to perform a conservative decision, in MDPs acting greedily may not be sufficient since an action can be ``safe'' in a single step but lead to a potentially dangerous state space where it will not be possible to satisfy the conservative constraint. 
Moreover, after $t$ steps, the action followed by the learning agent may lead to a state that is possibly different from the one observed by following the baseline.
This dynamical aspect is not captured by the bandit problem and should be explicitly taken into account by the learning agent in order to perform a meaningful decision.
This, together with the problem of counterfactual reasoning in an unknown MDP, make the conservative exploration problem is much more difficult (and interesting) in RL than in bandits.

This paper aims to provide the first analysis of conservative exploration in RL.
In Sec.~\ref{sec:conservativeexploration.rl} we explain the design choices that lead to the definition of the conservative condition for RL (both in average reward and finite horizon settings), and discuss all the issues introduced by the dynamical nature of the problem.
Then, we provide the first algorithm for efficient conservative exploration in average reward and analyze its regret guarantees.
The variant for finite-horizon problems is postponed to the appendix.
We conclude the paper with synthetic experiments.

\vspace{-0.1in}
\section{Preliminaries}
\vspace{-0.1in}
\label{sec:preliminaries}
We consider a Markov Decision Process~\citep[Sec. 8.3]{puterman1994markov} $M = ( \calS, \A, p, r )$ with state space $\calS$ and action space $\A$.
Every state-action pair $(s,a)$ is characterized by a reward distribution with mean $r(s,a)$ and support in $[0, \rmaxbound]$, and a transition distribution $p(\cdot|s,a)$ over next states.
We denote by $S = |\calS|$ and $A = |\A|$ the number of states and action
A stationary Markov randomized policy $\pi : \calS \rightarrow P(\A)$ maps states to distributions over actions.
The set of stationary randomized (resp. deterministic) policies is denoted by $\SR$ (resp. $\SD$).
Any policy $\pi \in \SR$ has an associated \emph{long-term average reward} (or gain) and a \emph{bias function} defined as
\begin{align*}
        g^\pi(s) &:= \lim_{T\to +\infty} \mathbb{E}^\pi_{s} \bigg[ \frac{1}{T}\sum_{t=1}^T r(s_t,a_t)\bigg]~~\text{ and }~~\\
        h^\pi(s) &:= \underset{T\to +\infty}{C\text{-}\lim}~\mathbb{E}^\pi_s \bigg[\sum_{t=1}^{T} \big(r(s_t,a_t) - g^\pi(s_t)\big)\bigg],
\end{align*}
where $\mathbb{E}^\pi_{s}$ denotes the expectation over trajectories generated starting from $s_1 = s$ with $a_t \sim \pi(s_t)$.
The bias $h^\pi(s)$ measures the expected total difference between the reward and the stationary reward in \emph{Cesaro-limit} 
(denoted by $C\text{-}\lim$). 
We denote by $sp(h^\pi) := \max_s h^\pi(s) - \min_{s} h^\pi(s)$ the \emph{span} (or range) of the bias function.
\begin{assumption}\label{asm:ergodic}
        The MDP $M$ is \emph{ergodic}.x
\end{assumption}
In ergodic MDPs, any policy $\pi \in \SR$ has \emph{constant} gain, \ie $g^{\pi}(s) = g^\pi$ for all $s\in\calS$.
There exists a policy $\pi^\star \in \argmax_\pi g^\pi$ for which $(g^\star,h^\star) = (g^{\pi^\star},h^{\pi^\star})$ satisfy the \emph{optimality equations},
\begin{equation*}\label{eq:optimality.equation}
        h^\star(s) + g^\star = L h^\star(s) := \max_{a \in \A} \{r(s,a) + p(\cdot|s,a)^\transp h^\star\},
\end{equation*}
where $L$ is the \emph{optimal} Bellman operator.
We use $D = \max_{s\neq s'} \min_{\pi \in\SD} \mathbb{E}[\tau_\pi(s'|s)]$ to denote the diameter of $M$, 
where $\tau_\pi(s'|s)$ is the hitting time of $s'$ starting from $s$.
We introduce the ``worst-case'' diameter
\begin{equation}\label{eq:worstdiameter}
        \worstD =  \max_{s\neq s'}\max_{\pi\in \SD} \mathbb{E}\left[ \tau_{\pi}(s'|s)\right],
\end{equation}
which defines the worst-case time it takes for any policy $\pi$ to move from any state $s$ to $s'$. Asm.~\ref{asm:ergodic} guarantees that $D \leq \worstD < \infty$.

\textbf{Exploration in RL.}
Let $M^\star$ be the true \emph{unknown} MDP. We consider the learning problem where $\mathcal{S}$, $\mathcal{A}$ and $\rmaxbound$ are \emph{known}, while rewards $r$ and transition probabilities $p$ are \emph{unknown} and need to be estimated online. 
We evaluate the performance of a learning algorithm $\mathfrak{A}$ after $T$ time steps by its cumulative \emph{regret}
\begin{equation}\label{eq:regret}
R(\mathfrak{A},T) = T g^\star - \sum_{t=1}^T r_t(s_t,a_t).
\end{equation}
The exploration-exploitation dilemma is a well-known problem in RL and (nearly optimal) solutions have been proposed in the literature both base on optimism-in-the-face-of-uncertainty~\citep[OFU, \eg][]{Jaksch10,Bartlett2009regal,fruit2018truncated} and Thompson sampling~\citep[TS, \eg][]{DBLP:conf/colt/GopalanM15, DBLP:journals/corr/OsbandR16a}.
Refer to~\citep{flp2019alttutorial} for more details.

\vspace{-0.1in}
\section{Conservative Exploration in RL}\label{sec:conservativeexploration.rl}
\vspace{-0.1in}

In conservative exploration, a learning agent is expected to perform as well as the optimal policy over time (i.e., regret minimization) under the \textit{constraint} that at no point in time its performance is significantly worse than a known baseline policy $\pi_b\in \SR$. This problem has been studied in the bandit literature~\citep{WuSLS16,KazerouniGAR17}, where the conservative constraint compares the cumulative \textit{expected} reward obtained by the actions $a_1,a_2,\ldots, a_t$ selected by the algorithm to the one of the baseline action $a_b$,
\begin{equation}\label{eq:condition.bandit}
\begin{aligned}
        \forall t>0, \quad \sum_{i=1}^t r(a_i) \geq (1-\alpha) \;t\; r(a_b),
\end{aligned}
\end{equation}
where $r(a)$ is the expected reward of action $a$. 
At any time $t$, conservative exploration algorithms first query a standard regret minimization algorithm (e.g., \ucb) and decide whether to play the proposed action $\wt{a}_t$ or the baseline $a_b$ based on the accumulated budget (\ie past rewards) and whether the estimated performance of $\wt{a}_t$ is sufficient to guarantee that the conservative constraint is satisfied at $t+1$ after $\wt{a}_t$ is executed. While~\eqref{eq:condition.bandit} effectively formalizes the objective of constraining an algorithm to never perform much worse than the baseline, in RL it is less obvious how to define such constraint. In the following we review three possible directions, we point out their limitations, and we finally propose a conservative condition for RL for which we derive an algorithm in the next section.

\textbf{Gain-based condition.} Instead of actions, RL exploration algorithms (e.g., \ucrl), first select a policy and then execute the corresponding actions. As a result, a direct way to obtain a conservative condition is to translate the reward of each action in~\eqref{eq:condition.bandit} to the \textit{gain} associated to the policies selected over time, i.e.,
\begin{equation}\label{eq:condition.gain}
\begin{aligned}
\forall t>0, \quad \sum_{i=1}^t g^{\pi_i} \geq (1-\alpha) \;t\; g^{\pi_b}.
\end{aligned}
\end{equation}
The main drawback of this formulation is that the gain $g^{\pi_i}$ is the expected \textit{asymptotic} average reward of a policy and it may be very far from the \textit{actual} reward accumulated while executing $\pi_i$ in the specific state $s_i$ achieved at time $i$. The same reasoning applies to the baseline policy, whose cumulative reward up to time $t$ may significantly different from $t$ times its gain. As a result, an algorithm that is conservative in the sense of~\eqref{eq:condition.gain} may still perform quite poorly in practice depending on $t$, the initial state, and the actual trajectories observed over time.

\textbf{Reward-based condition.} In order to address the concerns about the gain-based condition, we could define the stronger condition
\begin{equation}\label{eq:condition.reward}
\begin{aligned}
\forall t>0, \quad \sum_{i=1}^t r_i \geq (1-\alpha) \sum_{i=1}^t r_i^b,
\end{aligned}
\end{equation}
where $r_i$ is the sequence of rewards obtained while executing the algorithm and $r_i^b$ is the reward obtained by the baseline. While this condition may be desirable in principle (the learning algorithm never performs worse than baseline), it is impossible to achieve. In fact, even if the optimal policy $\pi^\star$ is executed for all $t$ steps, the condition may still be violated because of an unlucky realization of transitions and rewards. If we wanted to accounting for the effect of randomness, we would need to introduce an additional slack of order $O(\sqrt{t})$ (i.e., the cumulative deviation due to the randomness in the environment), which would make the condition looser and looser over time.

\textbf{Condition in expectation.} The previous remarks could be solved by taking the expectation of both sides 
\begin{equation}\label{eq:expected_cumrew_condition}
\begin{aligned}
\forall t>0, \;\;
& \mathbb{E}_{\mathfrak{A}}\left[
\sum_{i=1}^t r_i(s_i, a_i) \Big| s_1 = s
\right]\\
&  \geq (1-\alpha) \mathbb{E}\left[ 
\sum_{i=1}^t
r_i(s_i, a_i) \Big| s_1 = s, \pi_b
\right],
\end{aligned}
\end{equation}
where $\mathbb{E}_{\mathfrak{A}}$ denotes the expectation w.r.t.\ the trajectory of states and actions generated by the learning algorithm $\mathfrak{A}$, while the RHS is simply the expected reward obtain by running the baseline for $t$ steps. Condition~\eqref{eq:expected_cumrew_condition} effectively captures the nature of the RL problem w.r.t.\ the bandit case. In fact, after $t$ steps, the actions followed by the learning algorithm may lead to a state that is possibly very different from the one we would have reached by playing only the baseline policy from the beginning.
This deviation in the state dynamics needs to be taken into account when deciding if an exploratory policy is safe to play in the future. In the bandit case, selecting the baseline action contributes to \textit{build a conservative budget} that can be \textit{spent} to play explorative actions later on (i.e., by selecting $a_b$, the LHS of~\eqref{eq:condition.bandit} in increased by $r(a_b)$, while only a fraction $1-\alpha$ is added to the RHS, thus increasing the margin that may allow playing alternative actions later). In the RL case, selecting policy $\pi_b$ at time $t$ may not immediately contribute to increasing the conservative budget. In fact, the state $s_t$ where $\pi_b$ is applied may significantly differ from the state that $\pi_b$ \textit{would have achieved} had we selected it from the beginning. As a result, a conservative RL algorithm should be extra-cautious when selecting policies different from $\pi_b$ since their execution may lead to unfavorable states, where it is difficult to recover good performance, even when selecting the baseline policy.

While this may seem a reasonable requirement, unfortunately it is impossible to build an empirical estimate of~\eqref{eq:expected_cumrew_condition} that a conservative exploration algorithm could use to guide the choice of policies to execute. In fact, the LHS \textit{averages} the performance of the algorithm over multiple executions, while in practice we have only access to a single realization of the algorithm's process. This prevents from constructing accurate estimates of such expectation directly from the data observed up to time $t$. A possible approach would be to construct an estimate of the MDP and use it to \textit{replay} the algorithm itself for $t$ steps. Beside prohibitive computational complexity, the resulting estimate of the expected cumulative reward of $\mathfrak{A}$ would suffer from an error that increases with $t$, thus making it a poor proxy for~\eqref{eq:expected_cumrew_condition}.\footnote{More precisely, let $\wh M_t$ be an estimate of $M^\star$ and $\epsilon_t$ be the largest error in estimating its dynamics at time $t$. Estimating the expected cumulative reward by running (an infinite number of) simulations of $\mathfrak{A}$ in $\wh M_t$ would suffer from an error scaling as $t \epsilon_t$. For any regret minimization algorithm, $\epsilon_t$ cannot decrease linearly with $t$ and thus the estimation of $\mathbb{E}_{\mathfrak{A}}$ would have an error increasing with $t$.}

\textbf{Condition with conditional expectation.} Let $t$ be a generic time and $\mu_t = (\pi_1, \pi_2, \ldots,\pi_t)$, the non-stationary policy executed up to $t$. We require the algorithm to satisfy the following \emph{conditional} conservative condition
\begin{equation}\label{eq:cumrew_condition}
\begin{aligned}
\forall t>0, \;\;
&\mathbb{E}\left[
\sum_{i=1}^t r_i(s_i,a_i) |s_1=s, \mu_t
\right]\\
&\quad{}\geq (1-\alpha) \mathbb{E} \left[
\sum_{i=1}^t r_i(s_i, a_i) | s_1=s, \pi_b
\right].
\end{aligned}
\end{equation}
where the expectations are taken \wrt the trajectories generated by a \textit{fixed} non-stationary policy $\mu_t$ (i.e., we ignore how rewards affect $\mu_t$). Notice that this condition is now stochastic, as $\mu_t$ itself is a random variable and thus we require to satisfy~\eqref{eq:cumrew_condition} with \textit{high probability}. This formulation can be seen as relying on a \textit{pseudo}-performance evaluation of the algorithm instead of the actual expectation as in~\eqref{eq:expected_cumrew_condition}\footnote{We use \textit{pseudo}-performance to stress the link the \textit{pseudo}-regret formulations used in bandit~\citep[e.g.,][]{auer2002finite}} and it is similar to~\eqref{eq:condition.bandit}, which takes the expected performance of each of the (random) actions, thus ignoring their correlation with the rewards. This formulation has several advantages w.r.t.\ the conditions proposed above: \textbf{1)} it considers the sum of rewards rather than the gain as~\eqref{eq:condition.gain}, thus capturing the dynamical nature of RL, \textbf{2)} it contains expected values, so as to avoid penalizing the algorithm by unlucky noisy realizations as~\eqref{eq:condition.reward}, \textbf{3)} as shown in the next section, it can be verified using the samples observed by the algorithm unlike~\eqref{eq:expected_cumrew_condition}.

\paragraph{The finite-horizon case.}
We conclude the section, by reformulating~\eqref{eq:cumrew_condition} in the finite-horizon case. In this setting, the learning agent interacts with the environment in episodes of fixed length $H$. Let $\wb s$ be the initial state, $\pi_j$ be the policy proposed at episode $j$ and let $t = (k-1)H+1$ be beginning of the $k$-th episode. Then $\mu_t$ is a sequence of policies $\pi_j$, each executed for $H$ steps. In this case, condition~\eqref{eq:cumrew_condition} can be conveniently written as
\begin{equation}
    \label{eq:finite.horizon.condition}
    \begin{aligned}
    &(1-\alpha) k V^{\pi_b}_1(\wb s) \leq \mathbb{E}\left[
    \sum_{i=1}^t r_i(s_i,a_i) |s_1=s, \mu_t
    \right] \\
    &= \sum_{j=1}^k \mathbb{E}\left[
    \sum_{i=1}^H r_i^j(s_i^j,a_i^j) |s_1=s, \pi_j
    \right] = \sum_{j=1}^k V^{\pi_j}_1(\wb s)
    \end{aligned}
\end{equation}
where $V^{\pi}_1$ is the $H$ step value function of $\pi$ at the first stage. In this formulation, the conservative condition has a direct interpretation, as it directly mimics the bandit case~\eqref{eq:condition.bandit}. In fact, the performance of the algorithm up to episode $k$ is simply measured by the sum of the value functions of the policies executed over time (each for $H$ steps) and it is compared to the value function of the baseline itself.
Note that this definition is compatible with the regret: $R_{\textsc{FH}}(\mathfrak{A}, K) = \sum_{k=1}^K V^\star(\wb{s}) - V^{\pi_k}(\wb{s})$. Indeed, the regret defined in expectation w.r.t.\ the stochasticity of the model but not w.r.t.\ the algorithm, there is no expectation w.r.t.\ the possible sequence of policies generated by $\mathfrak{A}$.

\vspace{-0.1in}
\section{Conservative UCRL}
\vspace{-0.1in}
\label{sec:algorithm}
In this section, we introduce \emph{conservative upper-confidence bound for reinforcement learning} (\cucrl), an efficient algorithm for exploration-exploitation in average reward that both minimize the regret~\eqref{eq:regret} and satisfy condition~\eqref{eq:cumrew_condition}.


\begin{figure}[t]
\renewcommand\figurename{\small Figure}
\begin{minipage}{\columnwidth}
\bookboxx{
        \textbf{Input:} $\pi_b \in \SR$, $\delta \in (0,1)$, $r_{\max}$, $\mathcal{S}$, $\mathcal{A}$, $\alpha \in (0,1)$

        \noindent \textbf{For} episodes $k=1, 2, ...$ \textbf{do}
        \begin{enumerate}[leftmargin=4mm,itemsep=0mm]
                \item Set $t_k = t$ and episode counters $\nu_k (s,a) = 0$.
                \item Compute estimates $\wh{p}_k(s' | s,a)$, $\wh{r}_k(s,a)$ and a confidence set~$\mathcal{M}_k$.
                \item Compute an $\rmaxbound/\sqrt{t_k}$-approximation $\wt{\pi}_k$ of the optimistic planning problem $\max_{M \in \mathcal{M}_k, \pi \in \SD} \{g^\pi(M)\}$.
                \item Compute $(g^-_k, h^-_k) = \evi(\mathcal{L}_k^{\wt{\pi}_k}, \rmaxbound/\sqrt{t_k})$, see Eq.~\ref{eq:robust.bellman}.
                \item \textbf{if} Eq.~\ref{eq:condition.algorithm} is true \textbf{then} $\pi_k = \wt{\pi}_k$ \textbf{else} $\pi_k = \pi_b$
                \item Sample action $a_t \sim \pi_k(\cdot|s_t)$.
                \item \textbf{While} $\nu_k(s_t,a_t) \leq N_k^+(s_t,a_t) \wedge t \leq t_k + T_{k-1} $ \textbf{do}
                \begin{enumerate}[leftmargin=4mm,itemsep=0mm]
                        \item Execute $a_t$, obtain reward $r_{t}$, and observe $s_{t+1}$.
                        \item Set $\nu_k (s_t,a_t) = \nu_k (s_t,a_t) +  1$.
                        \item Sample action $a_{t+1} \sim \pi_k(\cdot|s_{t+1})$ and set $t = t+ 1$.
                \end{enumerate}
        \item Set $N_{k+1}(s,a) = N_{k}(s,a)+ \nu_k(s,a)$, $\Lambda_{k} = \Lambda_{k-1} \cup \{k\} \cdot \one{Eq.~\ref{eq:condition.algorithm}}$ and $\Lambda_{k}^c = \Lambda_{k-1}^c \cup \{k\} \cdot \one{\neg Eq.~\ref{eq:condition.algorithm}}$
        \end{enumerate}
}
\vspace{-0.2in}
\caption{\small \cucrl algorithm.}
\label{fig:cucrl}
\end{minipage}
\vspace{-0.15in}
\end{figure}

\cucrl builds on \ucrl in order to perform efficient \emph{conservative} exploration.
At each episode $k$, \cucrl builds a bounded parameter MDP $\mathcal{M}_k = \{ M = (\mathcal{S}, \mathcal{A}, {r}, {p}), {r}(s,a) \in B_r^k(s,a), {p}(\cdot|s,a) \in B_p^k(s,a)$, where $B_r^k(s,a) \in [0, \rmaxbound]$ and $B_p^k(s,a) \in \Delta_S$ are high-probability confidence intervals on the rewards and transition probabilities such that $M^\star \in \mathcal{M}_k$ w.h.p.\ and $\Delta_S$ is the $S$-dimensional simplex.
This confidence intervals can be built using Hoeffding or empirical Bernstein inequalities by using the samples available at episode $k$~\citep[\eg][]{Jaksch10, fruit2018constrained}.
\cucrl computes an optimistic policy $\wt{\pi}_k$ in the same way as \ucrl: $(\wt{M}_k, \wt{\pi}_k) \in \argmax_{M \in \mathcal{M}_k, \pi \in \SD} \{g^\pi(M)\}$. This problem can be solved using \evi (see Fig.~\ref{fig:evi} in appendix) on the optimistic optimal Bellman operator $\mathcal{L}^+_k$ of $\mathcal{M}_k$~\citep{Jaksch10}.\footnote{\label{foot:lopplus}$\mathcal{L}^+_k v(s) = \max_a\{ \max_{r \in B^r_k(s,a)}\{r\} + \max_{p\in B_k^p(s,a)} p^\transp v\}$.}
Then, it needs to decide whether policy $\pi_k$ is ``safe'' to play by checking a conservative condition $f_c(\mathcal{H}_k)$ (see Sec.~\ref{sec:conservative}) where $\mathcal{H}_k$ contains all the information (samples and chosen policies) available at the beginning of episode $k$, including the optimistic policy $\pi_k$.
If $f_c(\mathcal{H}_k)\geq 0$, the \ucrl policy $\wt\pi_k$ is ``safe'' to play and \cucrl plays $\pi_k = \wt\pi_k$ until the end of the episode.
Otherwise, \cucrl executes the baseline $\pi_b$, \ie $\pi_k = \pi_b$.
We denote by $\Lambda_k$ the set of episodes ($k$ included) where \ucrl executed an optimistic policy and by $\Lambda_k^c = \{1, \ldots, k\} \setminus \Lambda_k$ its complement.
Formally, if $f_c(\mathcal{H}_k) \geq 0$ we set $\Lambda_{k} = \Lambda_{k-1} \cup \{k\}$ else $\Lambda_k = \Lambda_{k-1}$.
The pseudocode of \cucrl is reported in Fig.~\ref{fig:cucrl}.

Note that, contrary to what happens in conservative (linear) bandits, the statistics of the algorithm are updated continuously, \ie using also the samples collected by running the baseline policy. This is possible since \ucrl is a model-based algorithm and any off-policy sample can be used to update the estimates of the model.
To have a better estimate of the conservative condition, it is possible to use the model available at episode $k$ to re-evaluate the policies $(\pi_l)_{l<k}$ at previuous episodes (change line 3 in Fig.~\ref{fig:cucrl}).
This will improve the empirical performance of \cucrl but breaks the regret analysis.

\vspace{-0.1in}
\subsection{Algorithmic Conservative Condition}
\vspace{-0.1in}
\label{sec:conservative}

We now derive a \emph{checkable} conservative condition that can be incorporated in the \ucrl structure illustrated in the previous section. In the bandit setting, it is relatively straightforward to turn~\eqref{eq:condition.bandit} into a condition that can be checked at any time $t$ using estimates and confidence intervals build from the data collected so far. On the other hand, while condition~\eqref{eq:cumrew_condition} effectively formalizes the requirement that the learning algorithm should constantly perform almost as well as the baseline policy, we need to consider the specific RL structure to obtain a condition that can be verified \textit{during the execution} of the algorithm itself. In order to simplify the derivation, we rely on the following assumption.

\begin{assumption}\label{asm:baseline}
The gain and bias function $(g^{\pi_b}, h^{\pi_b})$ of the baseline policy are known.
\end{assumption}

As explained in~\citep{KazerouniGAR17}, this is a reasonable assumption since the baseline policy is assumed to be the policy currently executed by the company and for which historical data are available.
We will mention how to relax this assumption in Sec.~\ref{sec:conservative}.

We follow two main steps in deriving a checkable condition. \textbf{1)} We need to estimate the cumulative reward obtained by each of the policies played by the learning algorithm directly from the samples observed so far. We do this by relating the cumulative reward to the gain and bias of each policy and then building their estimates. \textbf{2)} It is necessary to evaluate whether the policy proposed by \ucrl is safe to play w.r.t.\ the conservative condition, before actually executing it. While this is simple in bandit, as each action is executed for only one step. In RL, policies cannot be switched at each step and need to be played for a \textit{whole} episode. Nonetheless, the length of a \ucrl episode is not known in advance and this requires predicting for how long the explorative policy could be executed in order to check its performance.

%

\textbf{Step 1: Estimating the conditional conservative condition from data.}
In order to evaluate~\eqref{eq:cumrew_condition} from data, one may be tempted to first replace the sum of rewards obtained by each policy $\pi_j$ in $\mu_t$ on the lhs side by its gain $g^{\pi_j}$, similar to the gain-based condition in~\eqref{eq:condition.gain}. Indeed, under Asm.~\ref{asm:ergodic} any stationary policy $\pi$ receives \textit{asymptotically} an expected reward $g^\pi$ at each step. Unfortunately, in our case $\mathbb{E}\left[\sum_{i=1}^t r_i \big| \mu_t \right] \neq \sum_{j=1}^k T_{j} g^{\pi_{j}}$. In fact, when evaluating a policy for a finite number of steps, we need to account for the time required to reach the steady regime (\ie mixing time) and, as such, the influence of the state at which the policy is started. The notion of reward collected during the transient regime is captured by the bias function.\footnote{\citet[][Sec. 8.2.1]{puterman1994markov} refers to the gain as ``stationary'' reward while to the bias as ``transient'' reward.} In particular, for any stationary (unichain) policy $\pi \in \SR$ with gain $g^\pi$ and gain function $h^\pi$ executed for $t$ steps, we have that:
\begin{equation}\label{eq:cumrew_bellman}
\mathbb{E} \left[ \sum_{i=1}^t r_i \big|s_1 = s, \pi \right] = t\;g^\pi + h^\pi(s) - P_\pi^t (\cdot|s)^\transp h^\pi.
\end{equation}
As a result, we have the bounds
\begin{equation*}
t \; g^{\pi} - sp(h^{\pi}) \leq \mathbb{E} \left[ \sum_{i=1}^t r_i \big|s_1 = s, \pi \right] \leq t \; g^{\pi} + sp(h^{\pi}).
\end{equation*}
Leveraging prior knowledge of the gain and bias of the baseline, we can use the second inequality to directly upper bound the baseline performance as
\begin{equation}\label{eq:upper.baseline}
\mathbb{E}\left[
                        \sum_{i = 1}^{t} r_i \big| s_1 =s, \pi_b,
        \right]
        \leq sp(h^{\pi_b}) + t \; g^{\pi_b}.
\end{equation}
On the other hand, for a generic policy $\pi$, the gain and bias cannot be directly computed since $M^\star$ is unknown.
To estimate the cumulative reward of the algorithm we resort to the estimate of the true MDP build by \ucrl to construct a pessimistic estimate of the cumulative reward for any policy $\pi_j$ (\ie to perform counterfactual reasoning).


Given a policy $\pi$ and the bounded-parameter MDP $\mathcal{M}_k$, we are intersted in finding $\underline{g}^\pi$ such that:
        $\underline{g}^\pi :=  \min_{M \in \mathcal{M}_k} \{g^\pi(M)\}.$
Define the Bellman operator $\mathcal{L}_k^\pi$ associated to $\mathcal{M}_k$ as: $\forall v \in \mathbb{R}^S, \forall s \in \mathcal{S}$
\begin{equation}\label{eq:robust.bellman}
         \mathcal{L}_k^\pi v(s) := \min_{r \in B_r^k(s,a)} r + \min_{p \in B_p^k(s,a)} \{p^\transp v\}
\end{equation}
Then, there exists $(\underline{g}^\pi, \underline{h}^\pi)$ such that, $\forall s \in \mathcal{S}$, $\underline{g}^\pi e + \underline{h}^\pi = \mathcal{L}_k^\pi \underline{h}^\pi$ where $e = (1,\ldots, 1)$ (see Lem.~\ref{lem:robust}\emph{.1} in App.~\ref{sec:robust_pe}).
Similarly to what is done by \ucrl, we can use \evi with $\mathcal{L}_k^\pi$ to build an $\epsilon_k$-approximate solution of the Bellman equations.
Let $(g_n, v_n) = \evi(\mathcal{L}_k^\pi, \varepsilon_k)$, then $g_n - \varepsilon_k \leq \underline{g}^\pi \leq g^\pi(M^\star)$.
The values computed by the pessimistic policy evaluation can be then used to bound the cumulative reward of any stationary policy.

\begin{lemma} \label{eq:lowerbound.evi}
        Consider a bounded parameter MDP $\mathcal{M}$ such that $M^\star \in \mathcal{M}$ w.h.p.,
        a policy $\pi$ and let $(g_n, v_n) = \evi(\mathcal{L}^\pi, \varepsilon)$.
        Then, under Asm.~\ref{asm:ergodic} for any  state $s \in \mathcal{S}$:
        \[
                \mathbb{E}\left[\sum_{i=1}^t r_i
                        | s_1=s, \pi
                \right] \geq t (g_n - \varepsilon) - sp(v_n).
        \]
\end{lemma}

\textbf{Step 2: Test safety of optimistic policy.} Let $t_k$ be the time when episode $k$ starts. Policies $\pi_1,\ldots,\pi_{k-1}$ have been executed until $t_{k}-1$ and \ucrl  computed an optimistic policy $\wt\pi_k$. In order to guarantee that~\eqref{eq:cumrew_condition} is verified the algorithm needs to anticipate how well $\wt\pi_k$ may perform if executed for the next episode.
For any policy $(\pi_j)_{j<k} \cup \{\wt\pi_k\}$, we first compute $(g^-_j, h^-_j) = \evi(\mathcal{L}_j^{\pi_j}, \varepsilon_j)$.
\footnote{The subscript $j$ in the operator $\mathcal{L}_j^{\pi_j}$ denotes the fact that it is computed using the samples observed up to $t_j$. For each episode, we need to compute the estimate only for the new \ucrl policy. In order to have a tighter estimate of the conservative condition it possible to recompute the gain and bias of the past policies at every episode (or periodically) by using all the available samples (\ie using $\mathcal{L}^\pi_k$). However, this will break the current regret proof.}
If $\pi_j = \pi^b$ (i.e., the baseline was executed at episode $j$), we let $(g^-_j, h^-_j) = (g^{\pi_b}, h^{\pi_b})$ and $\varepsilon_j = 0$.
Then
\begin{equation}\label{eq:cutting_episodes}
        \begin{aligned}
                \mathbb{E}&\left[
                        \sum_{i = 1}^{t} r_i \big| s_1 =s, \mu_{t}
        \right]\\
                          &= \sum_{j = 1}^{k} \sum_{y\in \mathcal{S}}  \mathbb{P}\left( s_{t_{j}} = y \big| s, \mu_{t} \right)
                          \cdot \mathbb{E}\left[ \sum_{i = 1}^{T_{j}} r_{i}\mid y, \pi_j \right]\\
                          &\geq \sum_{j=1}^k T_j (g_j^- - \varepsilon_j ) - sp(h^-_j)
        \end{aligned}
\end{equation}
where $t_j$ is time at which episode $j$ started, $\mathbb{P}\left( s_{t_{j}} = y \big | s_1=s, \mu_{t}\right)$ is the probability of reaching state $y$ after $t_{j}$ steps starting from state $s$ following policy $\mu_{t}$. The inequality follows from Lem.~\ref{eq:lowerbound.evi}.
By lower bounding the LHS of~\eqref{eq:cumrew_condition} by~\eqref{eq:cutting_episodes} and upper bounding the RHS by~\eqref{eq:upper.baseline}, the conservative condition becomes:
\begin{align}
        \label{eq:condition.with.future}
        \sum_{j=1}^{k-1} \Big( T_j (g_j^- & - \varepsilon_j - g^{\pi_b} ) - sp(h^-_j) \Big)
           - sp(h^{\pi_b})\\
                                                     &\qquad{}+ T_k (g_k^- - \varepsilon_k - (1-\alpha) g^{\pi_b}) \geq 0 \nonumber
\end{align}
Note that the algorithm should check this condition at the beginning of episode $k$ in order to understand if the policy $\wt\pi_k$ is safe or if it should resort to playing policy $\pi_b$.
In many OFU algorithms, including \ucrl, the length of episode $k$ (\ie $T_k$) is not known at the beginning of the episode.
As a consequence, condition~\eqref{eq:condition.with.future} is not directly computable.
To overcome this limitation, we consider the dynamic episode condition introduced by~\citep{Ouyang2017learning}.
This stopping condition provides an upper-bound on the length of each episode as $T_k \leq T_{k-1} +1$, without affecting the regret bound of \ucrl (up to constants).
This condition can be used to further lower-bound the last term in \eqref{eq:condition.with.future} by
\begin{align}\label{eq:condition.algorithm.step}
&T_k (g_k^- - \varepsilon_k - (1-\alpha) g^{\pi_b}) \\
&\geq(T_{k-1} + 1)(g_k^- - \varepsilon_k - (1-\alpha) g^{\pi_b}) \cdot \one{(1-\alpha) g^{\pi_b} \geq g^-_k -\varepsilon_k}.\nonumber
\end{align}
Plugging this lower bound into~\eqref{eq:condition.with.future} gives the final conservative condition
\begin{align}\label{eq:condition.algorithm}
        &\sum_{j=1}^{k-1} \Big( T_j (g_j^-  - \varepsilon_j - g^{\pi_b} ) - sp(h^-_j) \Big)
- sp(h^{\pi_b}) +\nonumber\\
&(T_{k-1} + 1)(g_k^- - \varepsilon_k - (1-\alpha) g^{\pi_b}) \cdot \one{(1-\alpha) g^{\pi_b} \geq g^-_k -\varepsilon_k} \nonumber \\
&\geq 0,
\end{align}
tested by \cucrl at the beginning of each episode.
\textit{Unknown $(g^{\pi_b}, h^{\pi_b})$.} If the gain and bias of the baseline are unknown, we can use \evi on $\mathcal{L}^{+,\pi_b}_l$
(Eq.~\ref{eq:robust.bellman} with $\max$ instead of $\min$) to compute an optimistic estimate of the cumulative reward of the baseline up to time $t_l + T_{l-1} + 1$. While this account for the RHS of Eq.~\ref{eq:cumrew_condition}, we simply define $(g_l^-, h^-_l) = \evi(\mathcal{L}_l^{\pi_b}, \epsilon_l)$ for every episode $l \in \Lambda_{k-1}^c$ to compute a lower bound to the cumulative reward obtained by the algorithm by playing the baseline in episode before $k$.
Clearly, this approach is very pessimistic and it may be possible to design better strategies for this case.

\paragraph{The finite-horizon case.}
We conclude this section with a remark on the finite horizon case.
This case is much simpler and resemble the bandit setting.
We can directly build a lower bound $\wu{v}_{l,1}$ to the value function $V_1^{\pi_l}$ by using the model estimate and its uncertainty at episode $l$. This estimate can be computed via extended backward induction --see~\citep[][Alg. 2]{azar2017minimax}-- simply subtracting the exploration bonus, see Lem.~\ref{lem:value_function_pessimism_fh} in App.~\ref{app:finite.horizon}.
The same approach can be used to construct an optimistic and pessimistic estimate of $V^{\pi_b}_1$ when it is unknown.
This values can be directly plugged in~\eqref{eq:finite.horizon.condition} to define a checkable condition for the algorithm.

\subsection{Regret Guarantees}
\label{sec:regret}

We start providing an upper-bound to the regret of \cucrl showing the dependence on \ucrl and on the baseline $\pi_b$.
Since the set $\Lambda_{k}$ is updated at the end of the episode, we denote by $\Lambda_{T} = \Lambda_{k_T} \cup \{k_T\} \cdot \one{Eq.~\eqref{eq:condition.algorithm}}$ the set containing all the episodes where \cucrl played an optimistic policy. The set $\Lambda_{T}^c$ is its complement.
\begin{lemma}\label{lem:cucrl.first.ub}
        Under Asm.~\ref{asm:ergodic} and~\ref{asm:baseline}, for any $T$ and any conservative level $\alpha$, there exists a numerical constant $\beta > 0$ such that the regret of \cucrl is upper-bounded as
        \begin{align*}
                &R(\cucrl, T) \leq \beta \cdot \bigg( R_{\ucrl}(T|\Lambda_{T})  \\
                &\quad{} + \left(g^{\star} - g^{\pi_{b}}\right) \sum_{l \in \Lambda_T^c} T_l
                        + sp\big(h^{\pi_b}\big) \sqrt{SAT \ln(T/\delta)}
        \bigg),
        \end{align*}
        and the conservative condition~\eqref{eq:cumrew_condition} is met at every step $t=1,\ldots,T$ with probability at least $1-\frac{2\delta}{5}$.\footnote{The probability refers to both events: the regret bound and the conservative condition.}
\end{lemma}
$R_{\ucrl}(T|\Lambda_{T})$ denotes the regret of \ucrl over an horizon $T$ conditioned on the fact that the \ucrl policy is executed only at episodes $i \in \Lambda_T$. During the other episodes, the internal statistics of \ucrl are updated using the samples collected by the baseline policy $\pi_b$.
This does not pose any major technical challenge and, as shown in App.~\ref{app:proof_cucrl}, the \ucrl regret can be bounded as follows.
\begin{lemma}[\citep{Jaksch10}]\label{lem:ucrl.regret}
Let $L_T = \ln\left(\frac{5T}{\delta}\right)$, for any $T$, there exists a numerical constant $\beta> 0$ such that, with probability at least $1-\frac{2\delta}{5}$,
\begin{align*}
        R_{\ucrl}(T|\Lambda_{T}) \leq \beta DS\sqrt{AT L_T} +  \beta DS^{2}A L_T
\end{align*}
\end{lemma}
The second term in Lem.~\ref{lem:cucrl.first.ub} represents the regret incurred by the algorithm when playing the baseline policy $\pi_b$.
The following lemma shows that the total time spent executing conservative actions is sublinear in time (see Lem.~\ref{lem:nb_non_conservative_episodes} in App.~\ref{app:proof_cucrl} for details).
\begin{lemma}\label{lem:baseline.regret}
        For any $T>0$ and any conservative level $\alpha$, with probability at least $1 -\frac{2\delta}{5}$, the total number of play of conservative actions is bounded by:
\begin{align*}
   &\sum_{l\in\Lambda_{T}^c} T_{l} \leq 2\sqrt{SAT\ln(T)} +
   \frac{112SAL_{T}}{(\alpha g^{\pi_{b}})^{2}}(1 + S(D+\worstD)^{2})\\
   &+ \frac{16\sqrt{TL_{T}}}{\alpha g^{\pi_{b}}}\Big[ (D + \worstD)\sqrt{SA}  + \rmaxbound +  \sqrt{SA}sp(h^{\pi_{b}}) \Big]
\end{align*}
where $L_T = \ln\left( \frac{5SAT}{\delta} \right)$ and $\worstD < \infty$ as in Eq.~\ref{eq:worstdiameter}
\end{lemma}
\begin{proof}
        Let $\tau$ be the last episode played conservatively: $\tau = \sup\{k>0 : k \in \Lambda_{k}^c\}$.
        This means that at the beginning of episode $\tau$ the conservative condition was not verified. By rearranging the terms in Eq.~\ref{eq:condition.algorithm} and using simple bounds, we can write that:
        \[
                 \Delta_t + 4 k_T \left( sp(g^{\pi_b}) + \worstD + (1-\alpha)\rmaxbound \right) \geq \alpha \sum_{l=1}^{\tau-1} T_l g^{\pi_b}
        \]
        where $\Delta_\tau := \sum_{l \in \Lambda_{t-1}} T_l (\wt{g}_l - \underline{g}_l)$ and $(\wt{g}_l, \underline{g}_l)$ are optimistic and pessimistic gain of policy $\pi_l = \wt\pi_l$.
        Note that both satisfies the Bellman equation: $\wt g_l +\wt h_l = \mathcal{L}^+_l \wt h_l$ (see footnote \ref{foot:lopplus}) and $\wu g_l +\wu h_l = \mathcal{L}^{\pi_l}_l \wu h_l$ (see Eq.~\ref{eq:robust.bellman}).
        At this point, the important terms in upper-bounding $\Delta_\tau$ are similar to the one analysed in \ucrl.
        In particular, we have a term depending on the confidence intervals $\Delta_{ci}^{l,t} := 2\beta_{r}^{l}(s_{t},a_{t})  + \beta_{p}^{l}(s_{t},a_{t})\big( sp\big(\wt{h}_{l}\big) + sp\big(\underline{h}^{\pi_{l}}\big)\big)$ and one depending on the transitions $\Delta_{p}^{l,t} :=  p^{\star}(\cdot|s_{t},a_{t})^{\transp}\left( \wt{h}_{l} + \underline{h}^{\pi_{l}}\right) -  (\wt{h}_{l}(s_{t+1}) - \underline{h}^{\pi_{l}}(s_{t+1}))$.
        Let $X = \sum_{l=1}^{\tau-1} T_l$.
        By using the definition of the confidence intervals, it is easy to show that $\sum_{l \in \Lambda_{\tau-1}} \sum_{t=t_l}^{t_{l+1}-1} \Delta_{ci}^{l,t} \lesssim \sqrt{SAT X} + (D+\worstD)\sqrt{S^2AX}.$
Define the $\sigma$-algebra based on past history at $t$: $\mathcal{F}_t = \sigma(s_1,a_1,r_1,\ldots, s_t, a_t, r_t, s_{t+1})$.
        The sequence $(\Delta_{p}^{l,t}, \mathcal{F}_t)_{l,t}$ is an MDS. Thus, using Azuma inequality we have that $\sum_{l \in \Lambda_{\tau-1}} \sum_{t=t_l}^{t_{l+1}-1} \Delta_{ci}^{l,t} \lesssim (D+\Upsilon) \sqrt{T}$.
        Putting everything together we have a quadratic form in $X$ and solving it we can write that $\alpha g^{\pi_{b}}X \lesssim b_{T} + S^{2}AL_{T}^{\delta}(D+\worstD)^{2}/(\alpha g^{\pi_{b}})$ where $b_T = \wt{O}((D+\worstD)\sqrt{SAT})$ (see App.~\ref{app:proof_cucrl}).
        The result follows noticing that $\sum_{l \in \Lambda_{\tau}^c}T_{l} \leq \sum_{l=1}^{\tau-1} T_l + T_\tau$.
\end{proof}


Combining the results of
Lem.~\ref{lem:ucrl.regret} and Lem.~\ref{lem:baseline.regret} into Lem.~\ref{lem:cucrl.first.ub} leads to an overall regret of order $\wt O(\sqrt{T})$, which matches the regret of \ucrl. This shows that \cucrl is able to satisfy the conservative condition without compromising the learning performance. Nonetheless, the bound in Lem.~\ref{lem:baseline.regret} shows how conservative exploration is more challenging in RL compared to the bandit setting. While the dependency on the conservative level $\alpha$ is the same, the number of steps the baseline policy is executed can be as large as $\wt O(\sqrt{T})$ instead of constant as in \cucb~\citep{WuSLS16}. Furthermore, Lem.~\ref{lem:cucrl.first.ub} relies on an ergodicity assumption instead of the much milder communicating assumption needed by \ucrl to satisfy Lem.~\ref{lem:ucrl.regret}. Asm.~\ref{asm:ergodic} translates into the bound through the ``worst-case'' diameter $\worstD$, which in general is much larger than the diameter $D$. This dependency is due to the need of computing a lower bound to the reward accumulated by policies $\pi_j$ in the past (see~Lem.~\ref{eq:lowerbound.evi}). In fact, \ucrl only needs to compute \textit{upper bounds} on the gain and the value function returned by \evi by applying the optimistic Bellman operator $\mathcal{L}^+_k$ has span bounded by the diameter $D$. This is no longer the case for computing pessimistic estimates of the value of a policy. Whether Asm.~\ref{asm:ergodic} and the worst-case diameter $\worstD$ are the unavoidable price to pay for conservative exploration in infinite horizon RL remains as an open question.

\paragraph{The finite-horizon case.}
App.~\ref{app:finite.horizon} shows how to modify \ucbvi~\citep{azar2017minimax} to satisfy the conservative condition in Eq.~\ref{eq:finite.horizon.condition}.
In this setting, it is possible to show (see Prop.~\ref{prop:conservative_fh_regret} in appendix) that the number of conservative episodes is simply logarithmic in $T = KH$. Formally, $|\Lambda_{T}^c| = O(H^5S^2A\ln(T/\delta)/(\alpha r_b(\wu{\Delta}_b +\alpha r_b))$ where $0 < r_b \leq r(s,a)$, for all $(s,a)$, and $\wu{\Delta}_b = \min_s \{V_1^\star(s) - V^{\pi_b}_1(s)\}$ is the optimality gap. This problem dependent terms resemble the one in the bandit analysis.
The regret of conservative \ucbvi is bounded by 
$\wt{O}(H\sqrt{SAT}+1/(\alpha r_{b}(\wu{\Delta}_{b} + \alpha r_{b})))$.

\vspace{-0.1in}
\section{Experiments}
\vspace{-0.1in}
\label{sec:experiments}

\begin{figure}[t]
        \vspace{-.1in}
        \centering
        \includegraphics[width=.24\textwidth]{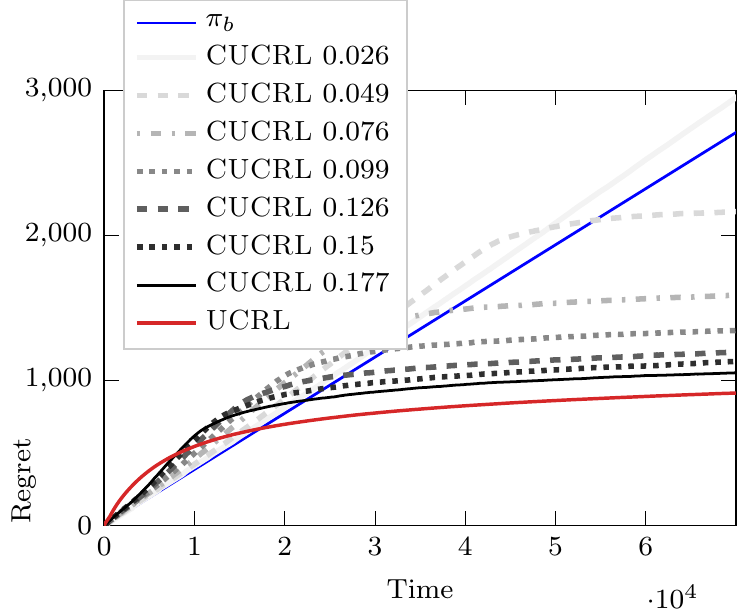}
        \includegraphics[width=.2\textwidth]{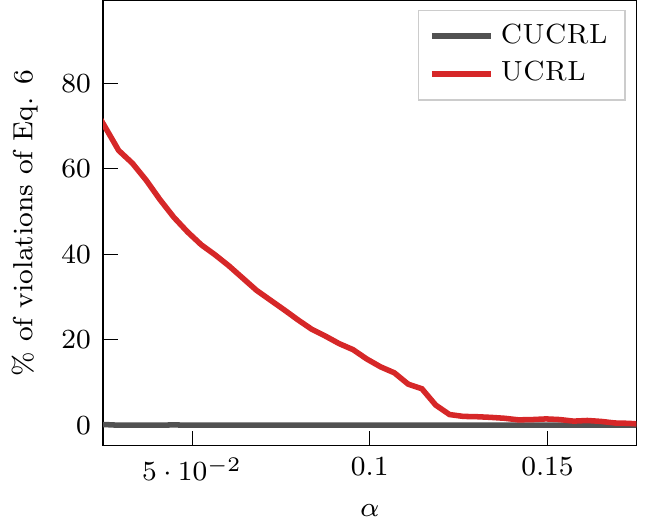}
        \vspace{-.1in}
        \caption{Inventory control problem.}
        \label{fig:sic_m6}
        \vspace{-.1in}
\end{figure}

In this section, we report results in the inventory control problem to illustrate the performance of \cucrl compared to unconstrained \ucrl and how it varies with the conservative level. See App.~\ref{app:experiments} for additional experiments for both average reward and finite horizon.
In order to have a better estimate of the budget, we re-evaluate past policies at each episode.
We start considering the \emph{stochastic inventory control} problem~\citep[][Sec. 3.2.1]{puterman1994markov} with capacity $M=6$ and uniform demand.
At the beginning of a month $t$, the manager has to decide the number of items to order in order to satisfy the random demand, taking into account the cost of ordering and maintainance of the inventory (see App.~\ref{app:experiments}).
Since the optimal policy is a threshold policy, as baseline we consider a $(\sigma, \Sigma)$ policy~\cite[][Sec. 3.2.1]{puterman1994markov} with target stock $\Sigma =4$ and capacity threshold $\sigma=4$.
Note that $g^\star=0.603$ and $(g^{\pi_b}, sp(h^{\pi_b}))=(0.565,0.651)$.
We use this domain to perform an ablation study \wrt the conservative level $\alpha$. 
We have taken $T=70000$ and the results are averaged over $100$ realizations.

Fig.~\ref{fig:sic_m6}(\emph{left}) shows that the regret of \cucrl grows at the same speed as the one of the baseline policy $\pi_b$ at the beginning (the conservative phase), because during this phase \cucrl is constrained to follow $\pi_b$ to make sure that constraint~\eqref{eq:cumrew_condition} is satisfied. Clearly, the duration of this conservative phase is proportional to the conservative level $\alpha$. As soon as \cucrl has built margin, it starts interleaving exploratory (optimistic) policies with the baseline.
After this phase, \cucrl has learn enough about the system and has a sufficient margin to behave as \ucrl.
As expected, Fig.~\ref{fig:sic_m6}(\emph{left}) confirms that the convergence to the \ucrl behavior happens more quickly for larger values of $\alpha$, \ie when the conservative condition is relaxed and \cucrl can explore more freely.
On the other hand, \ucrl converges faster since it is agnostic to the safety constraint and may explore very poor policies in the initial phase.
To better understand this condition, Fig.~\ref{fig:sic_m6}(\emph{right}) shows the percentage of time the constraint was violated in the first $15000$ steps (about $20\%$ of the overall time $T$).
\cucrl always satisfies the constraint for all values of $\alpha$ while \ucrl fails a significant number of times, especially when the conditions is tight (small values of $\alpha$).

\vspace{-0.1in}
\section{Conclusion}
\vspace{-0.1in}
\label{sec:conclusion}
We presented algorithms for conservative exploration for both finite horizon and average reward problems with $O(\sqrt{T})$ regret.
We have shown that the non-episodic nature of average reward problems makes the definition of the conservative condition much harder than in finite horizon problems.
In both cases, we used a model-based approach to perform counterfactual reasoning required by the conservative condition.
Recent papers have focused on model-free exploration in tabular settings or linear function approximation~\citep{qlearning2018,Yang2019matrix,jin2019lin}, thus a question is if it is possible for model-free algorithms to be conservative and still achieve  $\wt{O}(\sqrt{T})$ regret.


%

\bibliography{conservative,span}
\bibliographystyle{plainnat}

\appendix
\onecolumn
\section{Policy Evaluation with Uncertainties}
\label{sec:robust_pe}
\begin{figure}[t]
\renewcommand\figurename{\small Figure}
\begin{minipage}{\columnwidth}
\bookboxx{
        \textbf{Input:} Operator $\mathcal{L} : \mathbb{R}^S \to \mathbb{R}^S$ and accuracy $\epsilon > 0$  \\
        Set $v_0 = 0$, $v_1 = \mathcal{L} v_0$, $n=0$
        \begin{enumerate}[leftmargin=4mm,itemsep=0mm]
                \item \textbf{While} $sp(v_{n+1} - v_n) > \epsilon$ \textbf{do}
                \begin{enumerate}[leftmargin=4mm,itemsep=0mm]
                    \item $n = n+1$
                    \item $v_{n+1} = \mathcal{L} v_{n}$
                \end{enumerate}
        \item \textbf{Return:} $g_{n} = \frac{1}{2} \big(\max_s \{v_{n+1}(s) -v_{n}(s)\} + \min_{s} \{v_{n+1}(s) - v_{n}(s) \} \big)$ and $v_{n}$ 
        \end{enumerate}
}
\vspace{-0.2in}
\caption{\small \evi.}
\label{fig:evi}
\end{minipage}
\vspace{-0.1in}
\end{figure}
Consider a bounded parameter MDP $\mathcal{M}$ defined by a compact set $B_r(s,a) \subseteq [0, r_{\max}]$ and $B_p(s,a) \in \Delta_S$:
\begin{equation}\label{eq:bounded.parameter}
        \mathcal{M} = \{ M = (\mathcal{S}, \mathcal{A}, {r}, {p}), {r}(s,a) \in B_r(s,a), {p}(\cdot|s,a) \in B_p(s,a), \forall (s,a) \in \mathcal{S} \times \mathcal{A}\}
\end{equation}
In this paper, we consider confidence sets $B_r$ and $B_p$ that are polytopes.
We are interested in building a pessimistic (robust) estimate of the performance of a policy $\pi \in \SD$ in $\mathcal{M}$.
This robust optimization problem can be written as:
\begin{align}\label{eq:robust.problem}
    \underline{g}^\pi := \inf_{M \in \mathcal{M}} \{g^\pi(M)\}
\end{align}
where $g^\pi(M)$ is the gain of policy $\pi$ in the MDP $M$. Lemma~\ref{lem:robust} shows that there exists a solution to this problem that can be computed using \evi when the set $\mathcal{M}$ contains an ergodic MDP.

We recall that any bounded parameter MDP admits an equivalent representation as an extended MDP~\citep{Jaksch10} with identical state space $\mathcal{S}$ but compact action space.
For a deterministic policy  $\pi \in \SD$, the extended (pessimistic) Bellman operator $\mathcal{L}_\pi$ is defined as:
\begin{equation}\label{eq:extended.bellman}
        \forall v \in \mathbb{R}^S, \forall s \in \mathcal{S}, \quad \mathcal{L}_\pi v(s) := \min_{r \in B_r(s, \pi(s))} r + \min_{p \in B_p(s, \pi(s))} \{p^\transp v\}
\end{equation}

\begin{lemma}\label{lem:robust}
        Let $\mathcal{M}$ be a bounded-parameter MDP defined as in Eq.~\ref{eq:bounded.parameter} such that exists an \emph{ergodic} MDP $M \in \mathcal{M}$ w.h.p.
        Consider a policy $\pi \in \SD$, then:
        \begin{enumerate}
                \item There exists a tuple $(\wt{g}, \wt{h})\in \mathbb{R} \times \mathbb{R}^S$ such that:
                        \[
                                \forall s \in \mathcal{S}, \quad \wt{g} + \wt{h}(s) = \mathcal{L}_\pi \wt{h}(s)
                        \]
                        where $\mathcal{L}_\pi$ is the Bellman operator of the extended MDP $\mathcal{M}^+$ associated to $\mathcal{M}$ (see Eq.~\ref{eq:extended.bellman}).
                \item In addition, we have the following inequalities on the pair $(\tilde{g}, \tilde{h})$:
                    \begin{align*}
                    	 \tilde{g} \leq g^{\pi}(M) \qquad \text{and} \qquad sp(\tilde{h}) \leq \max_{\pi\in \Pi^{SD}(M)} \max_{s\neq s'} \mathbb{E}_{M}^{\pi}\left( \tau(s')|s\right) := \worstD < +\infty
                    \end{align*}
                    where $\mathbb{E}_{M}^{\pi}$ is the expectation of using policy $\pi$ in the MDP $M$ and $\tau(s')$ is the minimal number of steps to reach state $s'$.
        \end{enumerate}
\end{lemma}
\begin{proof}
        \textbf{Point 1.}
        We show that this policy evaluation problem is equivalent to a planning problem in an extended MDP $\mathcal{M}^-$ with negative reward.
        Consider the extended MDP $\mathcal{M}^- = (\mathcal{S}, \mathcal{A}^-, p^-, r^-)$ such that $\mathcal{A}^-_{s} = \{\pi(s)\} \times B_r(s,\pi(s)) \times B_p(s,\pi(s))$.
        For any state $s \in \mathcal{S}$ and action $a^- = (\pi(s), r(s,\pi(s)), p(\cdot|s,\pi(s))) \in \mathcal{A}_s$, 
        \begin{align*}
                r^-(s,a^-) &= {\color{blue}-r(s,\pi(s))}\\
                p^-(\cdot|s,a^-) &= p(\cdot|s,\pi(s))
        \end{align*}
        
        Denote by $\mathcal{L}^-$ the optimal Bellman operator of $\mathcal{M}^-$.
        Since $B_r(s, \pi(s))$ and $B_p(s, \pi(s))$ are polytopes, $\mathcal{L}^-$ can be interpreted as an optimal Bellman operator with finite number of actions.
        A sufficient condition for the existence of a solution of the optimality equations is that the MDP is weakly communicating~\citep[][Chap. 8-9]{puterman1994markov}.
        Note that $\mathcal{M}^-$ contains the model defined by $P^\pi$, \ie the Markov chain induced by $\pi$ in $M$.\footnote{We abuse of language since $\mathcal{M}^-$ is not formally a set. We should formally refer to the bounded parameter MDP associated to $\mathcal{M}^-$, \ie built considering $B_p(s, \pi(s))$ and $B_r(s, \pi(s))$. Note that $p(\cdot|s, \pi(s)) \in B_p(s, \pi(s))$ w.h.p.}
        Since $P^\pi$ is ergodic, $\mathcal{M}^-$ is at least communicating and thus $\mathcal{L}^-$ converges to a solution of the optimality equations.
        Extended value iteration~\citep{Jaksch10} on $\mathcal{L}^-$ converges toward a gain and bias $(g^-, h^-)$ such that:
        \begin{equation*}\label{eq:extended.pess.mdp}
                \begin{aligned}
                g^- + h^-(s) 
                &= L^- h^-(s) = \max_{a\in\mathcal{A}^-_s} \{ r^-(s,a) + p^-(\cdot|s,a)^\transp h^-\}\\
                &= \max_{r \in B_r(s,\pi(s))} \{-r\} + \max_{p \in B_p(s,\pi(s))} p^\transp h^-\\
                &= -\min \{B_r(s,\pi(s))\} +\max_{p \in B_p(s,\pi(s))} p^\transp h^-
                \end{aligned}
        \end{equation*}
        By rearranging, we have that:
        \begin{align*}
                -g^- + (-h^-)(s) &= \min \{B_r(s,\pi(s))\} + \min_{p \in B_p(s,\pi(s))} p^\transp (-h^-)\\
                                 &= \mathcal{L}_\pi (-h^-)(s)
        \end{align*}
        Thus follows that $\wt{g} = -g^-$ and $\wt{h} = -h^-$.
        This shows the relationship between maximizing over policies in the extended MDP $\mathcal{M}^-$ and minimizing over the set of models induced by $\pi$.
        
        \textbf{Point 2.} Let's begin by bounding the span of the bias $\tilde{h}$. Thanks to Theorem $4$ of \cite{Bartlett2009regal}, we have that the span of $\tilde{h}$ is upper-bounded by the diameter of the extended MDP $\mathcal{M}^{-}$, i.e:
        \begin{align*}
        sp\big(\tilde{h}\big)\leq  \max_{s\neq s'}\inf_{\pi^{-}\in \Pi^{SD}\left(\mathcal{M}^{-}\right)} \mathbb{E}_{\pi^{-}}\left( \tau(s')|s\right)
        \end{align*} 
        where $\mathbb{E}_{\pi^{-}}$ is the expectation of using policy $\pi^{-}$ in the extended MDP $\mathcal{M}^{-}$ and  $\tau(s')$ is the hitting time of state $s'$. But let's define the policy $\pi^{\star}$ in the extended MDP $\mathcal{M}^{-}$ such that for a state $s$, it chooses the action:
        \begin{align*}
        \pi^{\star}(s) = (\pi(s), r^{\star}(s,\pi(s)), p^{\star}(.|s,\pi(s)))
        \end{align*}
        with $r^{\star}$ and $p^{\star}$ the true parameter of the MDP $M$, this is possible because w.h.p the MDP $M^{\pi} \in \mathcal{M}^{-}$ with $M^{\pi}$ the Markov chain induced by using policy $\pi$ in the MDP $M$. Thus for any pair of states $(s,s')$:
        \begin{align*}
        \mathbb{E}_{\pi^{\star}}\left( \tau(s')|s\right) = \mathbb{E}_{M}^{\pi}(\tau(s')|s)
        \end{align*}
        with $\mathbb{E}_{M}^{\pi}$ the expectation of using policy $\pi$ in the MDP $M$. Therefore:
        \begin{align*}
         sp(\tilde{h})&\leq  \max_{s\neq s'}\inf_{\pi^{-}\in \Pi^{SD}\left(M^{-}\right)} \mathbb{E}_{\pi^{-}}\left( \tau(s')|s\right) \\
         &\leq \mathbb{E}_{M}^{\pi}(\tau(s')|s) \leq \worstD:= \max_{\pi\in \Pi^{SD}(M)} \max_{s\neq s'} \mathbb{E}_{M}^{\pi}\left( \tau(s')|s\right)
        \end{align*}
        And $\worstD < +\infty$ because $M$ is assumed to be ergodic.
        
        Let's show that the gain $\tilde{g}$ is a lower bound on the gain of the policy $\pi$ in the MDP $M$. Indeed, because the operator $\mathcal{L}^{-}$ converges toward solution of the optimality equations for negative rewards, we have that, see~\cite[][Th. 8.4.1]{puterman1994markov}:
        \begin{align*}
        g^{-} \geq -g^{\pi}(M)
        \end{align*}
        because reversing the sign of the rewards in the MDP $M$ changes the sign of the gain of a policy. Thus, $\tilde{g} \leq g^{\pi}(M)$.
\end{proof}
As a consequence, we can use \evi on $\mathcal{L}_\pi$ to compute a solution for problem~\ref{eq:robust.problem}.
\evi generates a sequence of vectors $(v_i)$ such that $v_{i+1} = \mathcal{L}_\pi v_i$ and $v_0 = 0$.
If the algorithm is stopped when $sp(v_{n+1} - v_n) \leq \epsilon$ we have~\citep[][Sec. 8.3.1]{puterman1994markov} that:
\begin{equation}\label{eq:evi_guarantees}
        |g_n - \wt{g}| \leq \epsilon/2 \quad \text{ and } \quad \|\mathcal{L}_\pi v_n - v_n - g_n e\|_{\infty} \leq \epsilon
\end{equation}
where $e = (1, \ldots, 1)$ and $g_n = \frac{1}{2} (\max_s \{v_{n+1}(s) - v_{n}(s)\} + \min_s \{ v_{n+1}(s) - v_n(s)\})$.
The following lemma shows how we can use the value produced by \evi to lower bound the expected sum of rewards under a policy $\pi$.
\begin{lemma}\label{lem:bound_evi}
        Let $(g_n, v_n)$ the values computed by \evi using $\mathcal{L}_\pi$ and an accuracy $\epsilon$.
        Then, the cumulative reward collected by policy $\pi$ in $M$ after $t$ steps can be lower bounded by:
        \[
                \forall y \in\mathcal{S}, \quad \mathbb{E}_M\left[ \sum_{i=1}^t r_i | s_1 = y, \pi \right] \geq t (g_n - \epsilon) - sp(v_n)
        \]
        In addition, 
         \[
         	sp(v_{n}) \leq \worstD
        \]
\end{lemma}
\begin{proof}
        Using the inequalities in~\eqref{eq:evi_guarantees} we can write that:
        \begin{align*}
                v_n(s) + g_n \leq \mathcal{L}_\pi v_n(s) &= \min_{r \in B_r(s, \pi(s))} r + \min_{p \in B_p(s, \pi(s))} \{p^\transp v\} + \epsilon\\
                                                         &\leq r(s,\pi(s)) + p(\cdot|s, \pi(s))^\transp v_n + \epsilon
        \end{align*}
        since $r(s,\pi(s)) \in B_r(s, \pi(s))$ and $p(\cdot|s, \pi(s)) \in B_p(s,\pi(s))$ w.h.p. 
        By iterating this inequality, we get that for all $t > 0$ and state $s$ :
\begin{align*}
        v_{n}(s) + tg_{n} \leq (t-1)\varepsilon + p^{t}(\cdot| s,\pi(s))^{\intercal} v_{n} + \mathbb{E}\left[ \sum_{i=1}^{t}  r_i\left(s_{i},\pi(s_{i})\right) | s_1=s\right]
\end{align*}
The statement follows by noticing that 
\[
        sp(v_n) = \max_s v_n(s) - \min_s v_n(s) 
        \geq \underbrace{p^{t}(\cdot| y,\pi(y))^{\intercal} v_{n}}_{\leq \max_s v_n(s)}
        - \underbrace{v_n(y)}_{\geq \min_s v_n(s)}, \quad \forall y \in \mathcal{S}
\]
The last statement is a direct consequence of the argument developed in section $4.3.1$ of \cite{Jaksch10}. This reasoning relies on the fact that the initial vector used in EVI is a zero span vector. 
\end{proof}

\section{Regret Bound for CUCRL}\label{app:proof_cucrl}
\begin{lemma}\label{lem:regret_decomposition}
        The regret of \cucrl can be upper-bounded for some $\beta > 0$, with probability at least $1-\frac{2\delta}{5}$, by:
        \begin{align*}
                R(\cucrl, T) \leq \beta \cdot \left( R(\ucrl, T|\Lambda_{T}) + \left(g^{\star} - g^{\pi_{b}}\right)
                \sum_{k \in \Lambda_T^c} T_k 
                + \max\{\rmaxbound, sp\left(h^{\pi_{b}}\right)\} \sqrt{SAT\ln(T/\delta)}
        \right)
        \end{align*}
\end{lemma}
\begin{proof}
Recall that $k_t = \sup\{k>0 : t> t_k\}$ is the episode at time $t$ and that the regret is defined as $R(\cucrl, T) = \sum_{t=1}^T \Big(g^\star - r_t(s_t,a_t)\Big)$.

Since the baseline policy $\pi_b$ may be stochastic, as a first step we replace the observed reward by its expectation.
As done in~\citep{fruit2018constrained} we use Azuma's inequality that gives, with probability at least $1- \frac{\delta}{5}$:
 \begin{align}\label{eq:bounding.rewards.expectation}
  \forall T \geq 1,~~ -\sum_{t=1}^T r_t \leq -\sum_{t=1}^T \sum_{a \in \A} \pi_{k_t}(s_t,a) r(s_t,a) + 2\rmaxbound\sqrt{T\ln \left(\frac{5T}{\delta}\right)}
 \end{align}

 We denote by $\Lambda_T = \Lambda_{k_T} \cup \{k_T\} \cdot \one{Eq.~\ref{eq:condition.algorithm}}$ the set of episodes where the algorithm played an UCRL policy. Note that we cannot directly consider $\Lambda_{k_T}$ since the set is updated at the end of the episode and the last episode may not have ended at $T$. Similarly we denote by $\Lambda_T^c = \Lambda_{k_T}^c \cup \{k_T\} \cdot \one{\neg Eq.~\ref{eq:condition.algorithm}}$.
Then, the regret of \cucrl can be decomposed as follow:
\begin{equation}\label{eq:regret_decomposition_avg_rwd}
        \begin{aligned}
                R(\cucrl,T) &= \sum_{t = 1}^{T} \left( g^{\star} - \sum_{a \in \A} \pi_{k_t}(s_t,a) r(s_t,a) \right) + 2\rmaxbound\sqrt{T\ln \left(\frac{5T}{\delta}\right)}\\
                    &= 2\rmaxbound\sqrt{T\ln \left(\frac{5T}{\delta}\right)} +
                    \underbrace{
                            \sum_{k = 1}^{k_T}\one{k \in \Lambda_{T}} \sum_{t = t_{k}}^{t_{k+1}-1} (g^{\star} - r(s_{t},a_{t}))
                            }_{:= R(\ucrl, T|\Lambda_{T})}\\
                    &\quad{}+ \sum_{k = 1}^{k_T}\one{k\in \Lambda_{T}^{c}}\Bigg(
                     \left(g^{\star} - g^{\pi_{b}}\right)(t_{k+1} - t_{k}) 
+
\underbrace{
            \sum_{t = t_{k}}^{t_{k+1}-1} \left( g^{\pi_{b}} - \sum_{a \in \A} \pi_{b}(s_t,a) r(s_t,a) \right)  }_{:= {\color{blue}\Delta^c_k}}  \Bigg)
        \end{aligned}
\end{equation}
Moreover, note that the \ucrl policy is deterministic so we hate that $\sum_{a \in \A} \pi_{k_t}(s_t,a) r(s_t,a) = r(s_t,a_t)$ when $k_t \in \Lambda_{T}$.  
The second term, denoted $R(\ucrl,T|\Lambda_{k_T})$, is the regret suffered by \ucrl over $\sum_{k \in \Lambda_{k_T}} T_k$ steps. 
The only difference with the orginal analysis~\citep{Jaksch10} is that the confidence intervals used by UCRL are updated when using the baseline policy, however it does not affect the regret of UCRL because it only means the confidence intervals used shrinks faster for some state-action pairs.
We will analyze this term in Lem.~\ref{lem:regret_ucrl}.
To decompose $\Delta^c_k$ we can use the Bellman equations ($g^{\pi_b} e = L^{\pi_b} h^{\pi_b} - h^{\pi_b}$):
\begin{align*}
        \sum_{k \in \Lambda_{T}^c} {\color{blue}\Delta^c_k}
        &=\sum_{k\in \Lambda_{T}^{c}} \sum_{t = t_{k}}^{t_{k+1}-1}  \sum_a \pi_b(s_t, a) p(\cdot|s_{t},a)^{\intercal}h^{\pi_{b}} - h_{\pi^{b}}(s_{t}) \\
        &= \sum_{k\in \Lambda_{T}^{c}} \sum_{t = t_{k}}^{t_{k+1}-1}  
        \underbrace{
\sum_{a \in \A} \pi_b(s,a)
\Big(p(\cdot|s_{t},a)^{\intercal}h^{\pi_{b}} \Big) - h^{\pi_{b}}(s_{t+1})}_{:= {\color{Purple}\Delta_{k,t}^{c,p}}}
+ \sum_{k\in \Lambda_{T}^{c}} \underbrace{\sum_{t = t_{k}}^{t_{k+1}-1} \left( h^{\pi_{b}}(s_{t+1}) - h^{\pi_{b}}(s_{t})\right)}_{:={\color{CadetBlue}\Delta_{k}^{c,2}}}
\end{align*}
But, $\Delta_{k}^{c,2}$ can be bounded using a telescopic sum argument and the number of episodes:
\begin{align*}
\sum_{k\in \Lambda_{T}^{c}} {\color{CadetBlue}\Delta_{k}^{c,2}}
=  \sum_{k\in \Lambda_{T}^{c}} h^{\pi_{b}}(s_{t_{k+1}}) - h^{\pi_{b}}(s_{t_{k}}) \leq |\Lambda_{T}^{c}|sp\left( h^{\pi_{b}}\right)
\end{align*}
Then it is easy to see that $(\Delta_{k,t}^{c,p})_{k,t}$ is a Martingale Difference Sequence with respect to the filtration $(\mathcal{F}_{t})_{t\in \mathbb{N}}$ which is generated by all the randomness in the environment and in the algorithm up until time $t$: $| \Delta_{k,t}^{c,p}| \leq 2 \| h^{\pi_b}\|_{\infty} \leq 2 sp(h^{\pi_b})$ and $\mathbb{E}[\Delta_{k,t}^{c,p}| \mathcal{F}_t] = 0$. Thus with probability $1 - \frac{\delta}{5}$:
\begin{align*}
\sum_{k\in \Lambda_{k_T}^{c}} \sum_{t = t_{k}}^{t_{k+1}-1} {\color{Purple}\Delta_{k,t}^{c,p}} \leq 4 sp\left(h^{\pi_{b}}\right)\sqrt{T\ln\left(\frac{5T}{\delta}\right)}
\end{align*}
Therefore putting all the above together, we have that with probability at least $1- \frac{2\delta}{5}$:
\begin{align*}
 R(\cucrl,T) \leq 2\rmaxbound\sqrt{T\ln \left(\frac{5T}{\delta}\right)} + R(\ucrl, T|\Lambda_{T}) + \left(g^{\star} - g^{\pi_{b}}\right)\sum_{k = 1}^{k_T}\one{k\in \Lambda_{T}^{c}}(t_{k+1} - t_{k})& \\
 + sp\left(h^{\pi_{b}}\right) \left( |\Lambda_{T}^{c}| + 4\sqrt{T\ln\left(\frac{5T}{\delta}\right)} \right)&
  \end{align*}
  
  As shown in~\citep[][Lem. 1]{Ouyang2017learning}, $k_{T} \leq \sqrt{2SAT\ln(T)}$ thus we can simply write that  $|\Lambda_{T}^{c}| \leq \sqrt{2SAT\ln(T)}$.
\end{proof}
In the next lemma, we bound the total number of steps where \cucrl used the baseline policy.
\begin{lemma}\label{lem:nb_non_conservative_episodes}
For any, $\delta>0$, the total length of episodes where the baseline policy is played by CUCRL after $T$ steps is upper-bounded with probability $1-2\delta/5$ by:
\begin{align*}
   \sum_{l\in\Lambda_{k_T}^{c}} T_{l} \leq &  2\sqrt{SAT\ln(T)} + \frac{16\sqrt{TL_{T}^{\delta}}}{\alpha g^{\pi_{b}}}\Big[ (D + \worstD)\sqrt{SA}  + \rmaxbound +  \sqrt{SA}sp(h^{\pi_{b}}) \Big] + \frac{112SAL_{T}^{\delta}}{(\alpha g^{\pi_{b}})^{2}}(1 + S(D+\worstD)^{2})
\end{align*}
with $L_{T}^{\delta} := \ln\left(\frac{5SAT}{\delta}\right)$ a logarithmic term in $T$.
\end{lemma}

\begin{proof}
        Let $\tau$ be the last episode played conservatively: $\tau = \sup\{k>0 : k \in \Lambda_{k}^c\}$.
        At the beginning of episode $\tau$ the conservative condition is not verified that is to say: 
\begin{align}\label{eq:conservative_check}
        \underbrace{\sum_{l\in \Lambda_{\tau-1}} T_{l}\left(g^{\pi_{b}} - g_{l}^{-} + \varepsilon_{l} \right)}_{:= {\color{Orange}\Delta_{\tau}^{1}}} + sp(h^{\pi_{b}}) \left( |\Lambda^{c}_{\tau-1}| + (1-\alpha)\right)  + \sum_{l\in \Lambda_{\tau-1} \cup \{ \tau\}} sp\left( h_{l}^{-}\right) + \nonumber& \\
+\left( T_{\tau-1} +1\right)\left((1-\alpha)g^{\pi_{b}} - g_{\tau}^{-} + \epsilon_{\tau}\right)\mathds{1}_{\{ (1-\alpha)g^{\pi_{b}} \geq g_{\tau}^{-} + \epsilon_{\tau}\}}&
\geq \alpha \sum_{l = 1}^{\tau-1} T_{l} g^{\pi_{b}} 
\end{align}
Let's proceeding by analysing each term on the RHS of Eq.~\ref{eq:conservative_check}. First, we have that $|\Lambda^{c}_{\tau-1}| \leq k_{T} \leq \sqrt{2SAT\ln(T)}$, thus:
\begin{align}\label{eq:bound_span_baseline}
sp(h^{\pi_{b}}) \left( |\Lambda^{c}_{\tau-1}| + (1-\alpha)\right)  \leq \Big(\sqrt{2SAT\ln(T)} + 1\Big)sp(h^{\pi_{b}})
\end{align}
On the other hand, thanks to Lem.~\ref{lem:bound_evi}, we have:
\begin{align}\label{eq:bound_span_evi}
        \sum_{l\in \Lambda_{\tau-1} \cup \{ \tau\}} sp\left( h_{l}^{-}\right) \leq (|\Lambda_{\tau-1}| +1) \worstD \leq 2\sqrt{2SAT\ln(T)}\worstD
\end{align}
Before analysing $\Delta_{\tau}^{1}$, let's bound the contribution of episode $\tau$:
\begin{align}\label{eq:contribution_tau}
\left( T_{\tau-1} +1\right)\left((1-\alpha)g^{\pi_{b}} - g_{\tau}^{-} - \epsilon_{\tau}\right)\mathds{1}_{\{ (1-\alpha)g^{\pi_{b}} \geq g_{\tau}^{-} + \epsilon_{\tau}\}} \leq (1-\alpha)g^{\pi_{b}}k_{T} \leq (1-\alpha)\rmaxbound\sqrt{2SAT\ln(T)} 
\end{align}
where we used the fact that for all episode $k$, we have $T_{k} \leq k$. 
Indeed the dynamic episode condition is such that for an episode $k$, $T_{k} \leq T_{k-1} + 1$ thus by iterating this inequality, $T_{k} \leq T_{0} + k = k$. 
At this point using equations~\ref{eq:conservative_check},~\ref{eq:bound_span_evi} and~\ref{eq:contribution_tau} we have:
\begin{align*}
        {\color{Orange}\Delta_{\tau}^{1}} + \Big(\sqrt{2SAT\ln(T)} + 1\Big)sp(h^{\pi_{b}})  + 2\sqrt{2SAT\ln(T)}\worstD +  (1-\alpha)\rmaxbound\sqrt{2SAT\ln(T)}  \geq \alpha \sum_{l = 1}^{\tau-1} T_{l} g^{\pi_{b}}
\end{align*}
Let's finish by analysing ${\color{Orange}\Delta_{\tau}^{1}}$. Let's define the event, $ \Gamma= \Bigg\{ \exists T > 0, \exists k \geq 1, \text{ s.t } M \not\in \mathcal{M}_{k}\Bigg\}$, by definition of $B_{r}^{k}$ and $B_{p}^{k}$, $\mathbb{P}\left( \Gamma \right) \leq \delta/5$, see~\citep[][App. B.2]{flp2019alttutorial} for a complete proof. 
We have that on the event $\Gamma^c$, for any $l \in \Lambda_{\tau-1}$,  $(g^-_l, h^-_l) = \evi(\mathcal{L}^{\pi_l}_l, \varepsilon_l)$ is such that $|\underline{g}^{\pi_{l}} - g_{l}^{-}| \leq \epsilon_{l}$ (see App.~\ref{sec:robust_pe}) where $\underline{g}^{\pi_l}$ is the true gain: $\underline{g}^{\pi_l} + \underline{h}^{\pi_l} = \mathcal{L}^{\pi_l}_l \underline{h}^{\pi_l}$.
Thus, since $\varepsilon_l \leq \rmaxbound / \sqrt{t_l}$:

\begin{align*}
        {\color{Orange}\Delta_{\tau}^{1}} = \sum_{l\in \Lambda_{\tau-1}} T_{l}\left(g^{\pi_{b}} - g_{l}^{-} + \epsilon_{l}\right) &\leq  2\sum_{l\in \Lambda_{\tau - 1}} T_{l}\varepsilon_{l} + \sum_{l\in \Lambda_{\tau-1}}T_{l}(g^{\pi_{b}} - \underline{g}^{\pi_{l}}) \\
                                                                                                                  &  \leq 4\rmaxbound\sqrt{T} + \sum_{l\in \Lambda_{\tau-1}}T_{l}(\wt{g}_{l} - \underline{g}^{\pi_{l}}) 
\end{align*}
where $\wt{g}_{l}$ is the optimistic gain at episode $l$ (see~\cite{flp2019alttutorial}) thus the last inequality comes from $g^{\pi_{b}} \leq g^{\star} \leq \wt{g}_{l}$ for every episode $l$. We can also define the optimistic bias at episode $l$, $\wt{h}_{l}$, the pair $(\wt{g}_{l}, \wt{h}_{l})$ is such that:
\begin{align*}
        \forall s\in \calS,  \qquad \wt{g}_{l} + \wt{h}_{l}(s) := \mathcal{L}^+_l \wt{h}_l := \max_{a} \left\{ \max_{r \in B_r^l(s,a)} r + \max_{p \in B_p^l(s,a))} p^\transp \wt{h}_{l} \right\}
\end{align*}
Recall that $\wt{\pi}_l \in \SD$ is the optimistic policy at episode $l$ and when $l \in \Lambda_{\tau-1}$, $\pi_l = \wt{\pi}_l$. Then, by using Bellman equations:
\begin{align*}
        \sum_{l\in \Lambda_{\tau-1}}T_{l}(\wt{g}_{l} - \underline{g}^{\pi_{l}})  
&= \sum_{l\in \Lambda_{\tau-1}}\sum_{t  = t_{l}}^{t_{l+1}-1}(\wt{g}_{l} - \underline{g}^{\pi_{l}}) 
= \sum_{l\in \Lambda_{\tau-1}}\sum_{t  = t_{l}}^{t_{l+1}-1} (\underbrace{\mathcal{L}^+_l \wt{h}_l(s_t)}_{:=\mathcal{L}^{+,\pi_l}_l\wt{h}_l(s_t)}, - \wt{h}_l(s_t) - \mathcal{L}^{\pi_l}_l \underline{h}^{\pi_l}(s_t) + \underline{h}^{\pi_l}(s_t))\\
&= \sum_{l\in \Lambda_{\tau-1}}\sum_{t  = t_{l}}^{t_{l+1}-1} \max_{r\in B_{r}^{l}(s_{t},a_{t})} r - \min_{r\in B_{r}^{l}(s_{t},a_t)} r 
+ \max_{p \in B_p^l(s_t,a_t)} p^\transp \wt{h}_{l} - \min_{p \in B_p^{l}(s_t,a_t)} p^\transp \underline{h}^{\pi_l} - \wt{h}_l(s_t) + \underline{h}^{\pi_l}(s_t)\\ 
&\leq \sum_{l\in \Lambda_{\tau-1}}\sum_{t  = t_{l}}^{t_{l+1}-1} 2\max_{r\in B_{r}^{l}(s_{t},a_{t})} r + \max_{q \in B_{p}^{l}(s_{t},a_{t})} (q - p^{\star})^{\intercal}\wt{h}_{l}  \\
&\quad{}- \min_{q \in B_{p}^{l}(s_{t},a_t)} (q - p^{\star})^{\intercal}\underline{h}^{\pi_l} +  p^{\star}(\cdot|s_{t},a_{t})^{\transp}\left( \wt{h}_{l} - \underline{h}^{\pi_{l}}\right)  -  (\wt{h}_{l}(s_{t+1}) - \underline{h}^{\pi_l}(s_{t+1}))  \\
&\quad{}+ (\wt{h}_{l}(s_{t+1}) - \wt{h}_{l}(s_{t}) + \underline{h}^{\pi_l}(s_{t+1})- \underline{h}^{\pi_l}(s_{t}))
\end{align*}
where $p^{\star}$ is the transition probability of the true MDP, $M^{\star}$. By a simple telescopic sum argument, we have:
\begin{align*}
  \sum_{l\in \Lambda_{\tau-1}}\sum_{t  = t_{l}}^{t_{l+1}-1} \wt{h}_{l}(s_{t+1}) - \wt{h}_{l}(s_{t}) + \underline{h}^{\pi_l}(s_{t+1})- \underline{h}^{\pi_l}(s_{t})  = |\Lambda_{\tau-1}|\big( sp\big(\wt{h}_{l}\big) + sp\big(\underline{h}^{\pi_{l}}\big)\big)
\end{align*}
At this point we need to explicitly define the concentration inequality used to construct the confidence sets $B_r^l$ and $B_p^l$.
For every $(s,a)\in \calS\times\mathcal{A}$, we define $\beta_{l}^{k}(s,a)$ such that:
\begin{align*}
\forall l\geq 1, \qquad B_{r}^{l}(s,a)\subset [\wt{r}_{l}(s,a) - \beta_{r}^{l}(s,a), \wt{r}_{l}(s,a) + \beta_{r}^{l}(s,a)]
\end{align*}
where $\wh{r}_{l}(s,a)$ is the empirical average of the reward received when visiting the state-action pairs $(s,a)$ at the beginning of episode $l$. 
For every $(s,a)\in \calS\times\mathcal{A}$, we define $\beta_{p}^{l}(s,a)$ as:
\begin{align*}
        B_p^l(s,a) = \left\{p \in \Delta_S : \|p(\cdot|s,a) - \wh{p}_l(\cdot|s,a)\|_1 \leq \beta_p^l(s,a) \right\}
\end{align*}
with $\wh{p}_{l}$ is the empirical average of the observed transitions. 
Choosing those $\beta_{r}^{l}$ and $\beta_{p}^{l}$ is done thanks to concentration inequalities such that event $\Gamma^c$ holds with high enough probability. In the following, we use:
\begin{equation*}
\forall s,a \qquad \beta_{r}^{l}(s,a) = \sqrt{\frac{7SAL_{T}^{\delta}}{2\max\{1, N_{l}(s,a)\}}} \text{ and } \beta_{p}^{l}(s,a) = S\sqrt{\frac{14AL_{T}^{\delta}}{\max\{1, N_{l}(s,a)\}}}
\end{equation*}
where $L_{T}^{\delta} = \ln\left(\frac{5SAT}{\delta}\right)$. For other choices of $\beta_{r}^{l}$ and $\beta_{p}^{l}$ refer to \citep{flp2019alttutorial}.
Similarly to what done in~\citep[][Sec. 4.3.1 and 4.3.2]{Jaksch10}, by using Holder's inequality and recentering the bias functions, we write:
\begin{align*}
        \sum_{l\in \Lambda_{\tau-1}}T_{l}(\wt{g}_{l} - \underline{g}^{\pi_{l}})  
        \leq |\Lambda_{\tau-1}|\big( sp\big(\wt{h}_{l}\big) + sp\big(\underline{h}^{\pi_{l}}\big)\big)  + 
        \underbrace{\sum_{l\in \Lambda_{\tau-1}}\sum_{t  = t_{l}}^{t_{l+1}-1} 2\beta_{r}^{l}(s_{t},a_{t})  + \beta_{p}^{l}(s_{t},a_{t})\big( sp\big(\wt{h}_{l}\big) + sp\big(\underline{h}^{\pi_{l}}\big)\big)}_{:=(a)}&  \\
        + \underbrace{\sum_{l\in \Lambda_{\tau-1}}\sum_{t  = t_{l}}^{t_{l+1}-1}p^{\star}(\cdot|s_{t},a_{t})^{\transp}\left( \wt{h}_{l} + \underline{h}^{\pi_{l}}\right) -  (\wt{h}_{l}(s_{t+1}) - \underline{h}^{\pi_{l}}(s_{t+1}))}_{:=(b)}& 
\end{align*}
To finish, the proof of this lemma, we need to bound the term $(a)$ and $(b)$.
In the following, we use the fact that  $sp\big(\wt{h}_{l}\big) + sp\big(\underline{h}^{\pi_{l}}\big)\leq D + \worstD$ (see Lem.~\ref{lem:bound_evi}) and again that $|\Lambda_{\tau-1}|\leq k_T \leq \sqrt{2SAT\ln(T)}$.
Let's begin with $(a)$, by definition of the radius of the confidence sets, we have:
\begin{align*}
\sum_{l\in \Lambda_{\tau-1}}\sum_{t  = t_{l}}^{t_{l+1}-1} \beta_{r}^{l}(s_{t},a_{t}) &= \sqrt{\frac{7SAL_{T}^{\delta}}{2}}\sum_{l\in \Lambda_{\tau-1}}\sum_{t  = t_{l}}^{t_{l+1}-1} \sqrt{\frac{1}{\max\{1, N_{l}(s_{t}, a_{t})\}}}
\leq  \sqrt{\frac{7SAL_{T}^{\delta}}{2}}\sqrt{\sum_{l = 1}^{\tau-1} T_{l}}
\end{align*}
and,
\begin{align*}
        \sum_{l\in \Lambda_{\tau-1}}\sum_{t=t_l}^{t_{l+1}-1} \beta_{p}^{l}(s_{t}, a_{t}) \leq S\sqrt{14L_{T}^{\delta}A}\sqrt{\sum_{l = 1}^{\tau-1} T_{l}}
\end{align*}
The second term $(b)$ is easy to bound because it is a Martingale Difference Sequence with respect to the filtration generated by all the randomness in the algorithm and the environment before the current step. 
For any time $t$, the $\sigma$-algebra generated by the history up to time $t$ included is $\mathcal{F}_t = \sigma(s_1, a_1, r_1, \ldots, s_t, a_t, r_t, s_{t+1})$. 
Define $X_t = \one{k_t \in \Lambda_T} (p(\cdot|s_t,\pi_{k_t}(s_t))^\transp u_{k_t} - u_{k_t}(s_{t+1}))$ with $u_{k_t}=\wt{h}_{k_t} - \underline{h}^{\pi_{k_t}}$.
Since $\pi_{k_t}$ is $\mathcal{F}_t$ measurable, $E[X_t|\mathcal{F}_{t-1}] = 0$ and $|X_t| \leq 2(D+\worstD)$.
Then $(X_t,\mathcal{F}_{t})_t$ is an MDS and nothing change compared to the analysis of \ucrl.
Therefore using Azuma-Hoeffding inequality, we have, with probability $1-\frac{\delta}{5}$ that: 
\begin{align*}
(b) \leq 2(D + \worstD)\sqrt{2TL_{T}^{\delta}}
\end{align*}
\faExclamationTriangle~Algorithmically, it is possible to evaluate the gain of the policies played in the past episodes at the beginning of the current episode. While this will provide a better estimate for the conservative condition, it will break the MDS structure in \emph{(b)} since $\underline{h}^{\pi_l}$ will be not measurable w.r.t.\ $\mathcal{F}_l$ since it is computed with samples collected after episode $l$.
Thus putting the bound for $(a)$ and $(b)$ together, we have:
\begin{align*}
        \sum_{l\in \Lambda_{\tau-1}}T_{l}(\wt{g}_{l} - \underline{g}^{\pi_{l}})  \leq & \left( D + \worstD \right)\sqrt{2SAT\ln(T)} + \sqrt{14SAL_{T}^{\delta}}\sqrt{\sum_{l = 1}^{\tau-1} T_{l}}  + \big( D + \worstD\big)S\sqrt{14L_{T}^{\delta}A}\sqrt{\sum_{l = 1}^{\tau-1} T_{l}}\\
&+ 2(D + \worstD)\sqrt{2TL_{T}^{\delta}}
\end{align*}

That is to say,
\begin{align*}\label{eq:step_before_max}
& \tikz[baseline]{\node (AA) {};} 4\rmaxbound\sqrt{T}  +  \left( D + \worstD \right)\sqrt{2SAT\ln(T)} + 2(D + \worstD)\sqrt{2TL_{T}^{\delta}}\\
& +  \Big(\sqrt{2SAT\ln(T)} + 1\Big)sp(h^{\pi_{b}})  
+ \sqrt{2SAT\ln(T)}\worstD +  (1-\alpha)\rmaxbound\sqrt{2SAT\ln(T)}\tikz[baseline]{\node (BB) {};}\\
&\qquad{}\quad{}+ \sqrt{14SAL_{T}^{\delta}}\sqrt{\sum_{l = 1}^{\tau-1} T_{l}}  + \big( D + \worstD\big)S\sqrt{14L_{T}^{\delta}A}\sqrt{\sum_{l = 1}^{\tau-1} T_{l}}\geq \alpha \sum_{l = 1}^{\tau-1} T_{l} g^{\pi_{b}}
\end{align*}
\tikz[overlay]{
        \draw[black] ($(AA)+(-10pt, 14pt)$) rectangle ($(BB)+(10pt, -8pt)$);
        \node at ($(BB)+(25pt, 8pt)$) {$:=b_T$};
}
Rearranging the terms and calling $X = \sum_{l = 1}^{\tau-1} T_{l}$, we have: 
\begin{align*}
        \alpha  g^{\pi_{b}} X \leq b_{T} + \left(\sqrt{14SAL_{T}^{\delta}}  + \big( D + \worstD\big)S\sqrt{14L_{T}^{\delta}A} \right) \sqrt{X}
\end{align*}
We have a quadratic equation and thus:
\begin{align*}
\sum_{l = 1}^{\tau-1} T_{l}  \leq \frac{2b_{T}}{\alpha g^{\pi_{b}}} + \frac{56SAL_{T}^{\delta}}{(\alpha g^{\pi_{b}})^{2}}(2 + 2S(D+\worstD)^{2})
\end{align*}
Therefore, as $\tau$ is the last episode where \cucrl played the policy $\pi_{b}$, we have $\sum_{l\in\Lambda_{T}^{c}} T_{l} = \sum_{l\in\Lambda_{\tau}^{c}} T_{l}$.
Also, because of the condition on the length of an episode $T_{k} \leq k$ for every $k$, therefore:
\begin{align*}
\sum_{l\in \Lambda_{T}^{c}} T_{l} = \sum_{l\in \Lambda_{\tau}^{c}} T_{l} \leq \sum_{l = 1}^{\tau-1} T_{l} + T_{\tau} \leq k_{T} + \frac{2b_{T}}{\alpha g^{\pi_{b}}} + \frac{56SAL_{T}^{\delta}}{(\alpha g^{\pi_{b}})^{2}}(2 + 2S(D+\worstD)^{2})
\end{align*}
\end{proof}

The following lemma states the regret of the \ucrl algorithm conditioned on running only the episodes in the set $\Lambda_{T}$.

\begin{lemma}\label{lem:regret_ucrl}
For any $\delta>0$, we have that after $T$, the regret of \ucrl is upper bounded with probability at least $1-\delta/5$ by:
\begin{align*}
R(\ucrl, T|\Lambda_{T}) \leq \beta DS\sqrt{AT\ln\left(\frac{5T}{\delta}\right)} +  \beta DS^{2}A\ln\left(\frac{5T}{\delta}\right)
\end{align*}
with $\beta$ a numerical constant.
\end{lemma}
\begin{proof}
The same type of bound has been shown in numerous work before \cite{Jaksch10,flp2019alttutorial}, however the proof presented in those works can not be readily applied to our setting. Indeed, when the algorithm chooses to play the baseline policy for an episode, then the confidence sets used in \cucrl are updated for the state-action pairs encountered during this episode. However, in the classic proof for the \ucrl algorithm the confidence sets are the same between the end of one episode and the beginning of the next one are the same. This may not be the case for \cucrl. 

Fortunately, when using the baseline policy during an episode, the confidence sets for every state-action pairs are either the same as the previous episode or are becoming tighter around the true parameters of the MDP $M^{\star}$. Thus, proving Lemma~\ref{lem:regret_ucrl} is similar to the proof presented in \cite{flp2019alttutorial}, the only difference resides in bounding the sum, $\sum_{k\in \Lambda_{k_T}}\sum_{t = t_{k}}^{t_{k+1}-1} 1/\sqrt{N_{k}^{+}(s_{t},a_{t})}$, which is bounded by the square root of the total number of samples in the proof of \cite{flp2019alttutorial} whereas in the case \cucrl it is bounded by the square root of the total number of samples gathered while exploring the set of policies plus the number of samples collected while playing the baseline policies. Therefore, at the end of the day both quantities are bounded by a constant times the square root of $T$.

A doubt someone could have is on controlling the term
        \begin{align*}
                &\sum_{k = 1}^{k_T} \one{k \in \Lambda_T} \sum_{t=t_k}^{t_{k+1}-1} \left( p(\cdot|s_t,\pi_k(s_t))^\transp u_k - u_k(s_{t})\right)\\
                &= \sum_{k = 1}^{k_T} \one{k \in \Lambda_T} \sum_{t=t_k}^{t_{k+1}-1} \underbrace{ \left(  p(\cdot|s_t,\pi_k(s_t))^\transp u_k - u_k(s_{t+1})\right)}_{\Delta_k^p}\\
                &\quad{} + \sum_{k = 1}^{k_T} \one{k \in \Lambda_T} \sum_{t=t_k}^{t_{k+1}-1}  u_k(s_{t+1}) - u_k(s_{t})\\
                &= \sum_{k = 1}^{k_T} \one{k \in \Lambda_T} \Delta_k^{p}+ \underbrace{\big(u_k(s_{t_{k+1}}) - u_k(s_{t_{k}})\big)}_{\leq sp(w_k) \leq D} 
        \end{align*}
        For any time $t$, the $\sigma$-algebra generated by the history up to time $t$ included is $\mathcal{F}_t = \sigma(s_1, a_1, r_1, \ldots, s_t, a_t, r_t, s_{t+1})$. 
        Define $X_t = \one{k_t \in \Lambda_T} (p(\cdot|s_t,\pi_{k_t}(s_t))^\transp u_k - u_k(s_{t+1}))$.
        Since $\pi_{k_t}$ is $\mathcal{F}_t$ measurable, $E[X_t|\mathcal{F}_{t-1}] = 0$ and $|X_t| \leq 2D$.
        Then $(X_t,\mathcal{F}_{t})_t$ is an MDS and nothing change compared to the analysis of \ucrl.

\end{proof}

Finally, plugging Lemmas~\ref{lem:nb_non_conservative_episodes} and~\ref{lem:regret_ucrl} into Lem.~\ref{lem:regret_decomposition}, we have that there exists a numerical constant $C_{1}$ such that with probability $1-\delta$:
\begin{align*}
    R(\cucrl, T) \leq C_{1} \Bigg( DS\sqrt{ATL_{T}^{\delta}} + \left(g^{\star} - g^{\pi_{b}}\right)&\Bigg( \sqrt{SAT\ln(T)} + \frac{\sqrt{TSAL_{T}^{\delta}}}{\alpha g^{\pi_{b}}}\max\{ sp(h^{\pi_{b}}), D + \worstD\} \\
    &+ \frac{S^{2}AL_{T}^{\delta}}{(\alpha g^{\pi_{b}})^{2}}(D+\worstD)^{2}\Bigg) +  \max\{\rmaxbound, sp\left(h^{\pi_{b}}\right)\} \sqrt{SAT\ln(T/\delta)}\Bigg)
\end{align*}

\section{Conservative Exploration in Finite Horizon Markov Decision Processes}
\label{app:finite.horizon}
In this section, we show how the conservative setting can be applied to finite horizon MDPs. Let's consider a finite-horizon MDP~\citep[][Chp. 4]{puterman1994markov} $M = (\mathcal{S}, \mathcal{A}, p, r, H)$ with state space $\mathcal{S}$ and action space $\mathcal{A}$. Every state-action pair is characterized by a reward distribution with mean $r(s,a)$ and support in $[0, 1]$ and a transition distribution $p(\cdot|s,a)$ over next state.
We denote by $S = |\mathcal{S}|$ and $A = |\mathcal{A}|$ the number of states and actions, and by $H$ the horizon of an episode.
A Markov randomized decision rule $d :\mathcal{S} \to P(\mathcal{A})$ maps states to distributions over actions.
A policy $\pi$ is a sequence of decision rules, \ie $\pi = (d_1, d_2, \ldots, d_H)$.
We denote by $\MR$ (resp.\ $\MD$) the set of Markov randomized (resp.\ deterministic) policies.
The value of a policy $\pi \in \MR$ is measured trough the value function
\begin{align*}
        \forall t\in [H], \forall s\in \mathcal{S} \qquad V^{\pi}_{t}(s) &= \mathbb{E}^{\pi}\left[ \sum_{l = t}^{H} r_{l}(s_{l}, a_l)\mid s_{t} = s \right]
\end{align*}
where the expectation is defined \wrt the model and policy (\ie $a_l \sim d_l(s_l)$).
This function gives the expected total reward that one could get by following  policy $\pi$ starting in state $s$, at time $t$.
There exists an optimal policy $\pi^\star \in \MD$~\citep[][Sec. 4.4]{puterman1994markov} for which $V^\star_t = V_t^{\pi^\star}$ satisfies the \emph{optimality equations}:
\begin{equation}\label{eq:optimality_eq_fh}
        \forall t \in [H], \forall s \in \mathcal{S}, \qquad V^{\star}_{t}(s) = \max_{a \in \mathcal{A}}
        \left\{ r_{t}(s, a) + p(\cdot|s,a)^\transp V^{\star}_{t+1} \right\}
        := L^{\star}_{t}V^{\star}_{t}
\end{equation}
where $V^\star_{H+1}(s) = 0$ for any state $s \in \mathcal{S}$.
The value function can be computed using backward induction~\citep[\eg][]{puterman1994markov,bertsekas1995dynamic} when the reward and transitions are known.
Given a policy $\pi \in \MD$, the associated value function satisfies the \emph{evaluation equations} $V^{\pi}_{t}(s) := L^\pi_t V^\pi_{t+1}(s) = r(s, d_t(s)) + p(\cdot|s,d_t(s))^\transp V^\pi_{t+1}$. The optimal policy is thus defined as $\pi^\star = \argmax_{\pi \in \MD} \{ L^\pi_t V^\star_t \}$, $\forall t \in [H]$.

In the following we assume that the learning agent known $\mathcal{S}$, $\mathcal{A}$ and $r_{max}$, while the reward and dynamics are \emph{unknown} and need to be estimated online.
Given a finite number of episode $K$, we evaluate the performance of a learning algorithm $\mathfrak{A}$ by its cumulative regret
\begin{align*}
        R(\mathfrak{A}, K) = \sum_{k = 1}^{K}  V^{\star}_{1}(s_{k,1}) - V^{\pi_{k}}_{1}(s_{k,1})
\end{align*}
where $\pi_{k}$ is the policy executed by the algorithm at episode $k$.

\paragraph{Conservative Condition}

Designing a conservative condition, in this setting is much easier than in the average reward case as evaluating a policy can be done through the value function which gives an estimation of the expected reward over an episode. Thus, we can use this evaluation of a policy to use in place of rewards in the bandits condition. Formally, denote by $\pi_b\in \MR$ the baseline policy and assume that $V_t^{\pi_b}$ is known.
In general, this assumption is not restrictive since the baseline performance can be estimated from historical data.
Given a conservative level $\alpha \in (0,1)$, we define the conservative condition as:
\begin{equation}\label{eq:conservative_cond_fh}
        \forall k \in [K], \quad \sum_{l = 1}^{k} V^{\pi_{l}}_{1}(s_{l,1}) \geq (1 - \alpha) \sum_{l=1}^k V^{\pi_{b}}_1(s_{l,1}) \qquad \text{ w.h.p }
\end{equation}
where $\pi_l$ is the policy executed by the algorithm at episode $l$ and $s_{l,1}$ is the starting state of episode $l$ before policy $\pi_l$ is chosen. The initial state can be chosen arbitrarily but should be revealed at the beginning of each episode. Note that this condition is random due the choice of the policies $(\pi_{l})_{l}$ and also because of the starting states thus the condition is required to hold with high probability.

Note that Eq.~\ref{eq:conservative_cond_fh} requires to evaluated the performance of policy $\pi_l$ on the true (unknown) MDP.
In order derive a practical condition, we need to construct an estimate of $V^{\pi_l}_1$.
In order to be conservative, we are interesting in deriving a lower bound on the value function of a generic policy $\pi$ which can be used in Eq.~\ref{eq:conservative_cond_fh}.

\paragraph{Pessimistic value function estimate.}
We recall that OFU algorithms (\eg \ucbvi and \euler) builds uncertainties around the rewards and dynamics that are used to perform an optimistic planning.
Formally, denote by $\wh{p}_k(\cdot|s,a)$ and $\wh{r}_k(s,a)$ the empirical transitions and rewards at episode $k$. Then, with high probability
\[
        |(p(\cdot|s,a) - \wh{p}_k(\cdot|s,a))^\transp v | \leq \beta_k^p(s,a) \quad \text{ and } \quad |r(s,a) - \wh{r}_k(s,a)| \leq \beta_k^r(s,a)
\]
for all $(s,a) \in \mathcal{S} \times \mathcal{A}$ and $v \in [0,H]^S$.
This uncertainties are used to compute an exploration bonus $b_k(s,a) = \beta_k^v(s,a) + \beta_k^r(s,a)$ that can be used to compute an optimistic estimate of the optimal value function.
Formally, at episode $k$, optimistic backward induction~\citep[\eg][Alg. 2]{azar2017minimax} computes an estimate value function $\bar{v}_{k,h}$ such that $\bar{v}_{k,h} \geq V^\star_t$ for any state $s$.
The same approach can be used to compute a pessimistic estimate of the optimal value function by subtracting the exploration bonus to the reward~\citep[\eg][]{zanette2019tight}.

The only difference in the conservative setting is that we are interesting to compute a pessimistic estimate for a policy different from the optimal one.
We thus define the \emph{pessimistic evaluation equations} for any episode $k$, step $h$, state $s$ and policy $\pi \in MR$ as:
\begin{equation}\label{eq:pess_pe_fh}
        \underline{v}^{\pi}_{k,h}(s) := \underline{L}_{k,h}^{\pi} \underline{v}_{k, h+1}^\pi
        = \sum_a \pi_{k,h}(s,a) \left(
                \wh{r}_{k}(s,a) - b_k(s,a) + \wh{p}_{k}(\cdot| s,a)^\transp \underline{v}_{k,h+1}^{\pi}
               \right)
\end{equation}
with $\underline{v}^{\pi}_{k,H+1}(s) = 0$ for all states $s \in \mathcal{S}$.
This value function is pessimistic (see Lem.~\ref{lem:value_function_pessimism_fh}) and can be computed using backward induction with $\underline{L}_k^\pi$.
\begin{lemma}\label{lem:value_function_pessimism_fh}
        Let $\pi = (d_1, \ldots, d_H) \in MR$ and $(\underline{v}_{k,h}^{\pi})_{h \in [H]}$ be the value function given by backward induction using Eq.~\ref{eq:pess_pe_fh} then with high probability:
    \begin{equation*}
            \forall (h,s)\in [H] \times\mathcal{S}, \qquad  V_{h}^{\pi}(s) \geq  \underline{v}^{\pi}_{k,h}(s)
    \end{equation*}
\end{lemma}
\begin{proof}

On the event that the concentration inequalities holds, let $\wh{r}_{k}(s,a)$ be the empirical reward  at episode $k$ and $\wh{p}_{k}(.|s,a)$ the empirical distribution over the next state from $(s,a)$ at episode $k$. We proceed with a backward induction. At time $H$ the statement is true. For $h<H$ :
\begin{align*}
        \underline{v}_{k,h}^{\pi}(s) - V_{h}^{\pi}(s)  &= \sum_a d_h(s,a) \left(\wh{r}_{k}(s,a) - b_k(s,a) + \wh{p}_{k}(\cdot| s,a)^\transp \underline{v}_{k,h+1}^{\pi} \right)  - L^{\pi}_{h} V_{h+1}^{\pi}(s) \\
                                     &=\sum_a d_h(s,a) \left(\underbrace{ \wh{r}_{k}(s,a) -r(s,a) - \beta_k^r(s,a)}_{\leq 0} \right)\\
                                     &\quad+ \sum_a d_h(s,a) \left( \wh{p}_k(\cdot|s,a)^\transp \underline{v}_{k,h+1}^\pi
                                     -p(\cdot|s,a)^\transp V^\pi_{h+1} - \beta_k^p(s,a)\right)\\
                                     &\leq \sum_a d_h(s,a) \left( \wh{p}_k(\cdot|s,a)^\transp \underline{v}_{k,h+1}^\pi
                                     -p(\cdot|s,a)^\transp V^\pi_{h+1} - \beta_k^p(s,a)\right)\\
                                     &\leq \sum_a d_h(s,a) \left( (\wh{p}_k(\cdot|s,a)
                                     -p(\cdot|s,a))^\transp V^\pi_{h+1} - \beta_k^p(s,a)\right) \leq 0
\end{align*}
where the first inequality is true because of the confidence intervals on the reward function and the penultimate inequality is true because of the backward induction hypothesis.
\end{proof}

Thanks to this result, we can formulate a condition that the algorithm can check, at the beginning of episode $k$ to decide if a policy is safe to play or not :
\begin{align}
\label{eq:finite_horizon_algorithm_condition}
\sum_{l\in \mathcal{S}_{k-1}\cup \{ k\}} \underline{v}_{l,1}^{\pi_{l}}(s_{l,1}) + \sum_{l\in \mathcal{S}_{k-1}^{c}} V_{l,1}^{\pi_{b}}(s_{l,1})\geq (1-\alpha) \sum_{l = 1}^{k} V_{l,1}^{\pi_{b}}(s_{l,1})
\end{align}
where $\calS_{k-1}$ is the set of episodes where the algorithm previously played non-conservatively, $\calS_{k-1}^{c} = [k-1] \setminus \calS_{k-1}$ is the set of episodes played conservatively and $(\pi_{l})_{l}$ is the policies that the OFU algorithm (\eg \ucbvi) would execute without the conservative constraint.

\begin{figure}[t]
\renewcommand\figurename{\small Figure}
\begin{minipage}{\columnwidth}
\bookboxx{
        \textbf{Input:} Policy $\pi_b$, $\delta \in (0,1)$, $r_{\max}$, $\mathcal{S}$, $\mathcal{A}$, $\alpha' \in (0,1)$, $H$ \\
		\textbf{Initialization:} Set $\mathcal{H} = \emptyset$, $\mathcal{S}_{0} = \emptyset$ and $\mathcal{S}_{0}^{c} = \emptyset$\\
        \noindent \textbf{For} episodes $k=1, 2, ...$ \textbf{do}
        \begin{enumerate}[leftmargin=4mm,itemsep=0mm]
                \item Compute optimistic policy $\pi_k$ using any OFU algorithm on history $\mathcal{H}$.
                \item Compute pessimistic estimate $\underline{v}_{k}^{\pi_k}$ as in Eq.~\ref{eq:pess_pe_fh}.
                \item \textbf{if} Equation (\ref{eq:finite_horizon_algorithm_condition}) not verified: \textbf{then}
                        \begin{enumerate}[leftmargin=1cm,itemsep=-1mm]
                                \item $\pi_k = \pi_b$, $\mathcal{S}_{k+1}^{c} = \mathcal{S}_{k}^{c} \cup \{ k \}$ and $\mathcal{S}_{k+1} = \mathcal{S}_{k}$
                        \end{enumerate}
                \textbf{else:}
                        \begin{enumerate}[leftmargin=1cm,itemsep=-1mm]
                                \item $\mathcal{S}_{k+1} = \mathcal{S}_{k} \cup \{ k \}$ and $\mathcal{S}_{k+1}^{c} = \mathcal{S}_{k}^{c}$
                        \end{enumerate}
                \item \textbf{for} $h = 1, \ldots, H$ \textbf{do}
                \begin{enumerate}[leftmargin=1cm,itemsep=0mm]
                        \item Execute $a_{k,h} = \pi_k(s_{k,h})$, obtain reward $r_{k,h}$, and observe $s_{k,h}$.
                        \item \textbf{if} $\pi_{k} \neq \pi_{b}$ \textbf{then:} add $(s_{k,h}, a_{k,h}, r_{k,h}, s_{k,h+1})$ to $\mathcal{H}$
                \end{enumerate}
        \end{enumerate}
}
\vspace{-0.2in}
\caption{\small \cucbvi algorithm.}
\label{fig:cucbvi}
\end{minipage}
\vspace{-0.1in}
\end{figure}

Alg.~\ref{fig:cucbvi} shows the generic structure of any conservative exploration algorithm for MDPs.
First, it computes an optimistic policy by leveraging on an OFU algorithm and the collected history.
Then it checks the conservative condition. When Eq.~\ref{eq:finite_horizon_algorithm_condition} is verified it plays the optimistic policy otherwise it plays conservatively by executing policy $\pi_b$.
This allows to build some budget for playing exploratory actions in the future.

\paragraph{Regret Guarantees}
We analyse Alg.~\ref{fig:cucbvi} with \ucbvi.
Before to introduce the upper-bound to the regret of \textsc{CUCB-VI} we introduce the following assumption on the baseline policy.
\begin{assumption}
        The baseline policy $\pi_b \in \MR$ is such  that $r_{b} := \min_{s} \{V^{\pi_{b}}_1(s)\} > 0$.
\end{assumption}
We can now state the main results:
\begin{proposition}\label{prop:conservative_fh_regret}
For $\delta>0$, the regret of conservative \ucbvi (\textsc{CUCB-VI}) is upper-bounded with probability at least $1-\delta$ by:
\begin{align}
R(\text{CUCB-VI}, K) \leq  &R(\text{UCB-VI}, K) + \frac{1}{4\alpha r_{b}(\underline{\Delta}_{b} + \alpha r_{b})}\Bigg( 16H^{3}L_{K}+ \big(200H^{5}S^{2}A + 128H^{5}SA\big)L_{K}^{2}  \Bigg)
\end{align}
where $L_{K} = \max\{\ln\left( 3KHSA/\delta\right), 1\}$ and $\underline{\Delta}_{b} = \min_{s\in \mathcal{S}} \{V^{\star}_1(s) - V^{\pi_{b}}_1(s)\}$.
\end{proposition}

\begin{proof}
Let's define the high probability event, $\mathcal{E}$, that is such that in this event, all the concentration inequalities holds and the Martingale Difference Sequence concentration inequalities also holds :

\begin{align*}
    \mathcal{E}_{1,\delta} := \bigcap_{(s,a) \in \mathcal{S}\times\mathcal{A}} \bigcap_{k\in [K]}\Bigg\{||p(.|s,a) - \hat{p}_{k}(.|s,a)||_{1} \leq \sqrt{\frac{2S\ln\left(3KSA/\delta\right)}{\max\{1, N_{k}(s,a)\}}}\Bigg\}& \\
    \bigcap \Bigg\{ |\hat{r}_{k}(s,a) - r(s,a)| \leq 2r_{\text{max}}\sqrt{\frac{\ln\left( 3KSA/\delta\right)}{\max\{1, N_{k}(s,a)\}}}\Bigg\}&
\end{align*}

\begin{align*}
    \mathcal{E}_{2,\delta} := \bigcap_{k\in [K]}\Bigg\{\sum_{l\in S_{k}}\sum_{h = 1}^{H} \varepsilon_{k,h} \leq H^{3/2}\sqrt{2\# S_{k} \ln\left( 3KH/\delta\right)}\Bigg\} \\
\end{align*}
and finally, $\mathcal{E} := \mathcal{E}_{1,\delta} \cap \mathcal{E}_{2,\delta}$, then $\mathcal{E}$ holds with probability at least $1-\delta$. Indeed,
 \begin{equation*}
 \mathbb{P}(\mathcal{E}^{c}) \leq \sum_{t = 1}^{HK}\frac{\delta}{3HK} + \sum_{s,a}\sum_{k} \frac{2\delta}{3KSA} \leq \delta
 \end{equation*}

 Under this event, we have that for all episode $k\in S_{K}$ :
 \begin{align*}
\wb{v}_{k,1}(s_{k,1}) - \underline{v}_{k,1}^{\pi_k}(s_{k,1}) \leq  \sum_{h  = 1}^{H}\varepsilon_{k,h} + 5\beta^{p}_{k}(s_{k,h}, d_{h}^{k}(s_{k,h})) + 2\beta^{r}_{k}(s_{k,h}, d_{h}^{k}(s_{k,h})) ,
 \end{align*}
 where $(\varepsilon_{k,h})_{k\in \calS_{K}, h\in [H]}$ is a martingale difference sequence with respect to the filtration $(\mathcal{F}_{k,h})_{k\in \calS_{K}, h\in [H]}$ that is generated by all the randomness before step $h$ of episode $k$. Indeed, for an episode $k$, let $\pi_{k} = (d_{1}^{k}, \ldots, d_{H}^{k})$, decomposing $\pi_{k}$ into successive decision rules.
  \begin{align*}
   \wb{v}_{k,1}(s_{k,1}) - \underline{v}_{k,2}^{\pi_k} \leq &2\beta^{r}_{k}(s_{k,1}, d_{1}^{k}(s_{k,1})) + \hat{p}_{k}(.\mid s_{k,1}, d_{1}^{k}(s_{k,1}))^{\intercal} (\wb{v}_{k,2} - \underline{v}_{k,2}^{\pi_k}) + 2\beta^{p}_{k}(s_{k,1}, d_{1}^{k}(s_{k,1}))
 \end{align*}
 Thus by defining, $B_{k,h} :=3\beta^{p}_{k}(s_{k,h}, d_{h}^{k}(s_{k,h}))  + 2\beta_{k}^{r}(s_{k,h},d_{h}^{k}(s_{k,h}))$, we have :
 \begin{align*}
    \wb{v}_{k,1}(s_{k,1}) - \underline{v}_{k,1}^{\pi_k}(s_{k,1}) \leq &B_{k,1} + (\hat{p}_{k}(.\mid s_{k,1}, d_{1}^{k}(s_{k,1}) - p(.|s_{k,1}, d_{1}^{k}(s_{k,1}))^{\intercal} (\wb{v}_{k,2} - \underline{v}_{k,2}^{\pi_k}) + (\wb{v}_{k,2}(s_{k,2}) - \underline{v}_{k,2}^{\pi_k}(s_{k,2}))\\
     &- (\wb{v}_{k,2}(s_{k,2}) - \underline{v}_{k,2}^{\pi_k}(s_{k,2}))  +  p(.|s_{k,1}, d_{1}^{k}(s_{k,1})^{\intercal} (\wb{v}_{k,2} - \underline{v}_{k,2}^{\pi_k}) \\
     &\leq   p(.|s_{k,1}, d_{1}^{k}(s_{k,1})^{\intercal} (\wb{v}_{k,2} - \underline{v}_{k,h}^{\pi_k}) - (\wb{v}_{k,2}(s_{k,2}) - \underline{v}_{k,2}^{\pi_k}(s_{k,2})) + 2\beta^{p}_{k}(s_{k,1}, d_{1}^{k}(s_{k,1})) + B_{k,1}\\
     &+ (\wb{v}_{k,2}(s_{k,2}) - \underline{v}_{k,2}^{\pi_k}(s_{k,2}))
 \end{align*}
 But let's define $\varepsilon_{k,h} :=  p(.|s_{k,h}, d_{h}^{k}(s_{k,h}))^{\intercal} (\wb{v}_{k,h} - \underline{v}_{k,h}^{\pi_k}) - (\wb{v}_{k,h}(s_{k,h+1}) - \underline{v}_{k,h}^{\pi_k}(s_{k,h+1}))$ then $(\varepsilon_{k,h})_{k\in [K], h\in [H]}$ is a Martingale Difference Sequence with respect to the filtration $\mathcal{F}_{k,h}$ which is generated by all the randomness in the environment and the algorithm before step $h$ of episode $k$ . Then, by recursion, we have :
 \begin{align*}
 \wb{v}_{k,1}(s_{k,1}) - \underline{v}_{k,1}^{\pi_k}(s_{k,1}) \leq  \sum_{h  = 1}^{H} B_{k,h} + \varepsilon_{k,h} + 2\beta^{p}_{k}(s_{k,h}, d_{h}^{k}(s_{k,h}))
 \end{align*}The regret of algorithm \cucbvi can be decomposed as :
\begin{align*}
R(\cucbvi, K) &= \sum_{k\in\mathcal{S}_{K}^{c}} V_{1}^{\star}(s_{k,1}) - V_{1}^{\pi_{b}}(s_{k,1}) + \sum_{k\in\mathcal{S}_{K}} V_{1}^{\star}(s_{k,1}) - V_{1}^{\pi_{k}}(s_{k,1}) \nonumber\\
                     &\leq |\mathcal{S}_{K}^{c}|\wb{\Delta}_{b} + R(\ucbvi,|\mathcal{S}_{K}|)
\end{align*}
where $\wb{\Delta}_{b} = \max_{s\in \mathcal{S}} V^{\star}(s) - V^{\pi_{b}}(s)$. Therefore bounding the regret amounts to bound the number of episode played conservatively. To do so,
  let's consider, $\tau$ the last episode played conservatively, then before the beginning of episode $\tau$, the condition \ref{eq:finite_horizon_algorithm_condition} is not verified and thus :
  \begin{align*}
      \alpha \sum_{k = 1}^{\tau} V_{1}^{\pi_{b}}(s_{k,1}) \leq \sum_{k\in \mathcal{S}_{\tau - 1} \cup \left\{\tau\right\}} \underbrace{V_{1}^{\pi_{b}}(s_{k,1}) - \underline{v}_{k,1}^{\pi_k}(s_{k,1})}_{ = \Delta_{k,1}}
  \end{align*}
Thus, let's finish this analysis by bounding $\Delta_{k,1} = V^{\pi_{b}}_{1}(s_{k,1}) - \underline{v}_{k,1}^{\pi_k}(s_{k,1})$ for all $k\in \calS_{K}$. But:
 \begin{align*}
 \Delta_{k,1} &= V_{1}^{\pi_{b}}(s_{k,1}) - V_{1}^{\star}(s_{k,1}) + V_{1}^{\star}(s_{k,1}) - \underline{v}_{k,1}^{\pi_k}(s_{k,1}) \leq -\underline{\Delta}_{b} +\wb{v}_{k,1}(s_{k,1}) - \underline{v}_{k,1}^{\pi_k}(s_{k,1}),
 \end{align*}
 where $\underline{\Delta}_{b} := \min_{s} V_{1}^{\star}(s) - V_{1}^{\pi_{b}}(s)$. Now, we need to bound the sum over all the non-conservative episodes of the difference between the optimistic and pessimistic value function. That is to say :
 \begin{align*}
 \sum_{l\in S_{\tau-1}}\sum_{h = 1}^{H} \beta_{k}^{r}(s_{k,h},d_{h}^{k}(s_{k,h})) = \sum_{l\in \calS_{\tau-1}}\sum_{h = 1}^{H} 2Hr_{\text{max}}\sqrt{\frac{2\ln\left(3KSA/\delta\right)}{\max\{1,N_{k}(s_{k,h},d_{h}^{k}(s_{k,h}))}} \\
 \leq 2r_{\text{max}}H^{2}\sqrt{2SAH|\calS_{\tau-1}|(1+\ln(|\calS_{\tau-1}|H))\ln\left(3KSA/\delta\right)}
 \end{align*}
 Also :
  \begin{align*}
 \sum_{l\in S_{\tau-1}}\sum_{h = 1}^{H} \beta_{k}^{p}(s_{k,h}, d_{h}^{k}(s_{k,h})) = \sum_{l\in S_{\tau-1}}\sum_{h = 1}^{H} H\sqrt{\frac{2S\ln\left(3KSA/\delta\right)}{\max\{1,N_{k}(s_{k,h},d_{h}^{k}(s_{k,h}))}}\\
 \leq H^{2}S\sqrt{2AH\# S_{\tau-1}(1+\ln(\# S_{\tau-1}H))\ln\left(3KSA/\delta\right)}
 \end{align*}
 and, under the event $\mathcal{E}$, $\sum_{l\in \calS_{\tau-1}} \sum_{h=1}^{H} \varepsilon_{k,h} \leq 2H^{3/2} \sqrt{2|\calS_{\tau-1}|\ln(3KH/\delta)}$. On the other hand, for the episode $\tau$, we can only bound the difference in value function by $H$. Finally, we have that $\tau = 1 + |\calS_{\tau-1}^{c}| + |\calS_{\tau-1}|$ and thus if we assume that $r_{b} := \min_{s}V^{\pi_{b}}(s) > 0$  :
 \begin{align*}
\alpha r_{b} (|\calS_{\tau-1}^{c}| + 1) \leq \alpha \sum_{k = 1}^{\tau} V_{1}^{\pi_{b}}(s_{k,1})  \leq -(\underline{\Delta}_{b} + \alpha r_{b})|\calS_{\tau-1}|+ 2H^{3/2} \sqrt{2|\calS_{\tau-1}|\ln(3KH/\delta)}& \\
+5H^{2}S\sqrt{2AH| \calS_{\tau-1}|(1+ \calS_{\tau-1}|H))\ln\left(3KSA/\delta\right)}& \\
+ 4r_{\text{max}}H^{2}\sqrt{2SAH|\calS_{\tau-1}|(1+\ln(|\calS_{\tau-1}|H))\ln\left(3KSA/\delta\right)}&
 \end{align*}
 Thus, the function on the RHS in bounded and using lemma $8$ of \cite{KazerouniGAR17}, we have :
 \begin{align*}
 \alpha r_{b} (| \calS_{\tau - 1}^{c}| + 1) \leq \frac{1}{4(\underline{\Delta}_{b} + \alpha r_{b})}\Bigg( 16H^{3}\ln\left( \frac{3KH}{\delta}\right) + \big(200H^{5}S^{2}A + 128r_{\text{max}}^{2}H^{5}SA\big)\times \nonumber\\
 \times(1 + \ln\left( HK\right))\ln\left(\frac{3KSA}{\delta}\right)  \Bigg)
 \end{align*} But by definition, $|\calS_{\tau-1}^{c}| + 1 = |\calS_{K}^{c}|$. Hence the result.
 \end{proof}

\paragraph{Experiments}

Finally, we end this presentation of conservativeness in finite horizon MDPs with some experiments. We consider a classic $3\times4$ gridworld problem with one goal state, a starting state and one trap state, we set $H = 10$, and the reward of any action in all the state to $-2$, the reward in the goal state to $10$ and the reward of falling in the trapping state to $-20$. We normalize the rewards to be in $[0,1]$. The baseline policy  is describing a path around the pit, see Fig~\ref{fig:baseline_policy}.
\begin{figure}[h]
\centering
\includegraphics[width=0.2\linewidth]{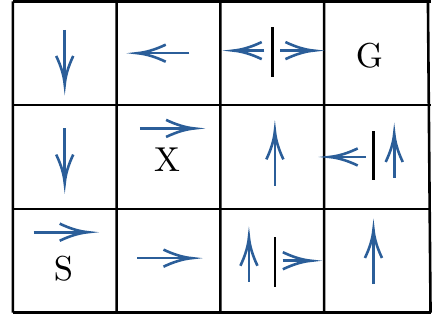}
\caption{Illustration of the baseline policy. S is the starting state, X is the pit a,d G is the goal state.}
\label{fig:baseline_policy}
\end{figure}
On the two position adjacent to the goal the baseline policy is stochastic with a probability of reaching the goal of $1/2$ for the position on the right of the goal and below the goal, respectively. On the last line the probability of going up or right is also uniform. Figure~\ref{fig:grid_fh} shows the impact of the conservative constraint on the regret of \ucbvi for a conservative coefficient $\alpha = 0.05$. Fig~\ref{fig:grid_fh} also shows the constraint as a function of the time for \ucbvi and \cucbvi that is to say: $\sum_{l=1}^{t} V^{\pi_{l}}(s_{0}) - (1-\alpha)V^{\pi_{b}}(s_{0})$ as a function of episode $t$ with $s_{0}$ the starting state of the gridworld. In the first $10\%$ episodes (i.e until episode $300$) the condition was violated by \ucbvi $83\%$ of the time.
\begin{figure}[h]
        \centering
        \includegraphics[width=.3\textwidth]{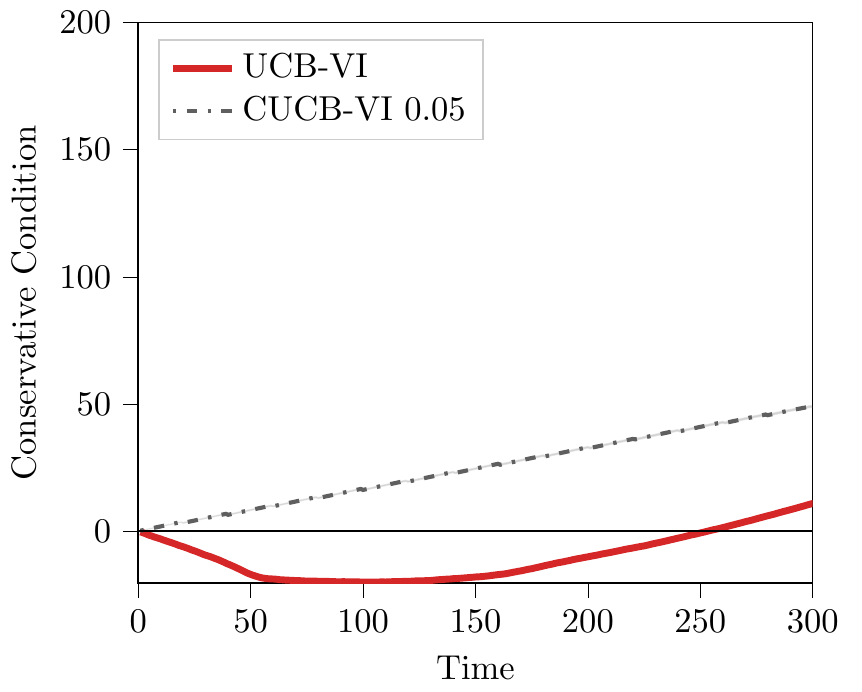}
        \includegraphics[width=.3\textwidth]{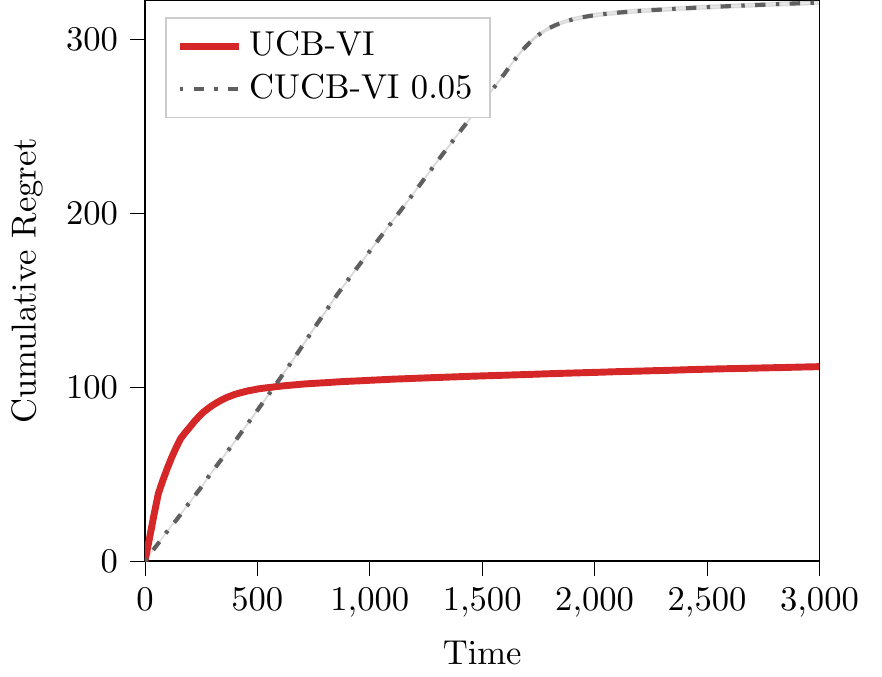}
        \vspace{-.1in}
        \caption{Regret and Conservative Condition for the gridworld problem}
        \label{fig:grid_fh}
        \vspace{-.1in}
\end{figure}

\section{Experiments}\label{app:experiments}
For average reward problems we consider ``simplified'' Bernstein confidence intervals given by:
\[
        \beta_r^k(s,a) = \sigma_r(s,a) \sqrt{\frac{\ln(SA/\delta)}{N_k^+(s,a)}} + \rmaxbound\frac{\ln(SA/\delta)}{N_k^+(s,a)} \quad \text{and} \quad \beta_{p}^k(s,a,s') = \sigma_p(s,a,s')\sqrt{\frac{\ln(SA/\delta)}{N_k^+(s,a)}} + \frac{\ln(SA/\delta)}{N_k^+(s,a)}
\]
where $N_k^+(s,a) = \max\{1, N_k(s,a)\}$, $\sigma_r(s,a)$ is the empirical standard deviation and $\sigma_p(s,a,s') = \sqrt{\wh{p}(s'|s,a) (1-\wh{p}(s'|s,a))}$.
\subsection{Single-Product Stochastic Inventory Control}
Maintaining inventories is necessary for any company dealing with physical products.
We consider the case of single product without backlogging.
The state space is the amount of products in the inventory, $\calS = \{0, \ldots, M\}$ where $M$ is the maximum capacity.
Given the state $s_t$ at the beginning of the month, the manager (agent) has to decide the amount of units $a_t$ to order.
We define $D_t$ to be the random demand of month $t$ and we assume a time-homogeneous probability distribution for the demand.
The inventory at time $t+1$ is given by
\[
        s_{t+1} = \max \{0, s_t + a_t - D_t \}
\]
The action space is $\A_s = \{0, \ldots, M-s\}$.
As in~\citep{puterman1994markov}, we assume a fixed cost $K>0$ for placing orders and a varible cost $c(a)$ that increases with the quantity ordered: $O(a) = \begin{cases}K+c(a) & a>0\\ 0 & \text{otherwise}\end{cases}$.
The cost of maintaining an inventory of $s$ items is defined by the nondecreasing function $h(s)$.
If the inventory is available to meet a demand $j$, the agent receives a revenue of $f(j)$.
The reward is thus defined as $r(s_t,a_t,s_{t+1}) = - O(a_t) - h(s_t + a_t) +f(s_t +a_t -s_{t+1})$. In the experiments, we use $K=4$, $c(x) = 2x$, $h(x) = x$ and $f(x) = 8x$.

In all the experiments, we normalize rewards such that the support is in $[0,1]$ and we use noise proportional to the reward mean: $r_t(s,a) = (1+c\eta_t) r(s,a)$ where $\eta_t \sim \mathcal{N}(0, 1)$ (we set $c = 0.1$).

\begin{figure}[t]
\centering
        \includegraphics[width=0.4\textwidth]{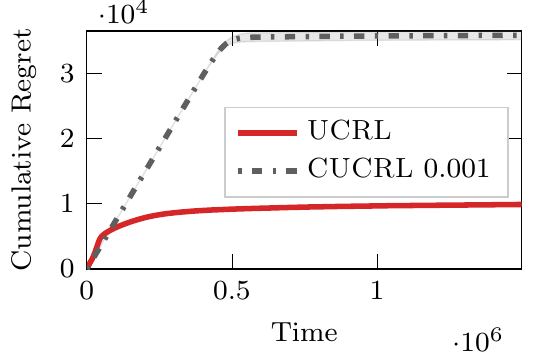}
        \caption{Regret of \ucrl and \cucrl on the Cost-Based Maintenance problem described in \ref{app:experiments}}
        \label{fig:cbm}
\end{figure}
\subsection{Cost-Based Maintenance}
The system is composed by $N$ components in an active redundant, parallel setting, which are subject to economic and stochastic dependence through load sharing. Each component $j \in [N]$ is described by its operational level $x_j = \{0, \ldots, L\}$. The level $L$ denotes that the component has failed. The deterioration process is modelled using a Poisson process. If all components have failed, the system is shut down and a penalty cost $p$ is paid. The replacement of a failed component cost $c_c$, while the same operation on an active component cost $c_p$ (usually $c_c \geq c_p$). There is also a fixed cost for maintenance $c_s$. At each time step, it is possible to replace simultaneously multiple components.
Please refer to~\citep{keizer2016cbm} for a complete description of dynamics and rewards.

We terminate the analysis of \cucrl with a more challenging test. We consider the condition-based maintenance problem~\citep[CBM,][]{keizer2016cbm} a multi-component system subject to structural, economic and stochastic dependences. We report a complete description of the problem in App.~\ref{app:experiments}.
The resulting MDP has $S = 121$ states and $A=4$ actions. The maintenance policy is often implemented as a threshold policy based on the deterioration level.
Such a threshold policy is not necessarily optimal for a system with economic dependence and redundancy.
We simulate this scenario by considering a strong (almost optimal) threshold policy for CBM without economic dependence as baseline. We make it stochastic by selecting with probability $0.3$ a random action. As a result we have that the optimal gain $g^\star=0.89$ while the baseline gain is $g^{\pi_b} = 0.82$.
Fig.~\ref{fig:cbm} shows the cumulative regret for \ucrl and \cucrl with $\alpha=0.001$.
\ucrl explores faster than \cucrl but violates the conservative condition $53\%$ of times in the initial phase (up to $t=140000$), incurring in multiple complete system failures.
On the other hand, \cucrl never violates the conservative condition.

\end{document}